%% file: main.tex
\ifdefined\FABLEDOCUMENTCLASSLOADED
\else
\documentclass[letterpaper]{article} 
\fi
\usepackage[preprint]{aaai2027}
\usepackage[hyphens]{url}  
\usepackage{graphicx} 
\usepackage{natbib}  
\usepackage{caption} 
\usepackage{algorithm}
\usepackage{algorithmic}
\usepackage{booktabs}
\usepackage{amsmath,amssymb,amsfonts,amsthm}
\usepackage{array}

\newtheorem{assumption}{Assumption}
\newtheorem{proposition}{Proposition}
\newtheorem{theorem}{Theorem}
\newtheorem{corollary}{Corollary}
\newtheorem{lemma}{Lemma}
\newtheorem{remark}{Remark}

\newcommand{\R}{\mathbb{R}}
\newcommand{\E}{\mathbb{E}}
\newcommand{\calA}{\mathcal{A}}
\newcommand{\calM}{\mathcal{M}}
\newcommand{\calT}{\mathcal{T}}
\newcommand{\calS}{\mathcal{S}}
\newcommand{\calK}{\mathcal{K}}
\newcommand{\feas}{\mathrm{feas}}
\newcommand{\cost}{\mathrm{cost}}
\newcommand{\argmax}{\operatorname*{arg\,max}}

\newcommand{\Normal}{\mathcal{N}}

\title{Personalizing Large Language Model Agents with Small Policy Models}
\author{
    Dian Jin\textsuperscript{\rm 1},
    Zhi Zhang\textsuperscript{\rm 2},
    Huichao Li\textsuperscript{\rm 1},
    Yihe Pan\textsuperscript{\rm 3},
    Rundong Huang\textsuperscript{\rm 1},
    Doudou Zhou\textsuperscript{\rm 1}\thanks{Corresponding author:
    \texttt{doudouzhou@nus.edu.sg}}
}
\affiliations{
    \textsuperscript{\rm 1}Department of Statistics and Data Science, National University of Singapore\\
    \textsuperscript{\rm 2}Department of Statistics and Data Science, University of California, Los Angeles\\
    \textsuperscript{\rm 3}Department of Data Science, Fudan University
}

\begin{document}

\maketitle

\begin{abstract}
Large language model (LLM) agents can retrieve memory, call tools, ask clarifying questions, and vary response style, yet adapting these execution decisions to an individual user remains difficult. Fine-tuning a separate LLM is costly or impossible for proprietary systems, while prompts and memory primarily expose user information to the agent rather than adapt its execution decisions from feedback. We formulate personalization of a frozen agent as online learning of a per-user execution policy from scalar feedback observed only for the executed action. We propose FABLE (Factorized Adaptive Bandit Layer for Execution), a lightweight policy layer outside a potentially black-box host agent. FABLE factorizes memory, information-acquisition, and response decisions so feedback updates related choices; filters actions through an externally specified feasible set before exploration; and learns user-specific residual preferences relative to a fixed default-and-cost score via Bayesian contextual Thompson sampling. Under a linear residual-reward model, FABLE with a theoretically calibrated
Gaussian sampling scale admits a
$\widetilde O(d^{3/2}\sqrt{n})$ regret guarantee against the best feasible
action at each round. On a four-domain tau2-bench evaluation, FABLE attains
the highest observed personalized reward and synthetic verbosity alignment
while tying the highest task-success mean. Matched comparisons support
factorization, onboarding, and online adaptation for preference-sensitive
metrics, but not task-success improvement.
An interactive demo is available at \url{https://fable-agent.github.io/}.

\end{abstract}

\section{Introduction}

For agents based on large language models (LLMs), personalization concerns both
the user information available to an agent and the execution policy governing
how it acts on that information. This policy determines how the agent uses
memory and tools, when it seeks clarification or confirmation, and how it
presents the final answer. In response to the same travel-planning request, one
user may expect the agent to retrieve relevant past trips, verify current
options online, and confirm before booking, whereas another may prohibit memory
access and prefer a direct answer without follow-up questions. Choosing between
these behaviors for a particular user is a decision problem that the host
model's capabilities alone do not settle.

Existing approaches leave this decision problem open in different ways.
Per-user fine-tuning adapts model parameters, but it is costly and unavailable
when the host is proprietary~\citep{tan2024oppu}. Profiles, prompts, and memory
systems supply user information to a fixed model, yet a rule stated in a prompt
is never revised when the user's later reactions contradict it, and these
mechanisms do not decide when the agent should consult memory, invoke a tool,
or ask a question~\citep{salemi2024lamp,chhikara2025mem0}. Learned decision
layers for frozen agents come closest: they select actions around a fixed host
from offline rollouts or online
feedback~\citep{yi2026harness,yu2026olivia}. However, they learn one shared
policy for task execution rather than a policy per user, and the online variant
conditions on the host's hidden states, which a black-box API does not expose.
Finally, recent personalized-agent methods adapt a single mechanism from
interaction, such as decoding, memory use, or tool
selection~\citep{qu2025tpop,liang2026pahf,yoon2026mpt}. In an agent, however,
the user reacts to the delivered interaction as a whole, and that reaction
reflects the combined effect of the memory, information-acquisition, and
response choices; a method that adapts one mechanism in isolation has no
defined way to attribute this single signal across the choices that produced
it.

This feedback structure is the central difficulty. Each interaction executes
one combination of execution decisions, and the user's reaction or the task
outcome is observed only for that combination; how the alternatives would have
performed is never revealed. Observations may therefore be frequent yet
individually uninformative about most of the action space. The combinations
nevertheless share components. Evidence that a user dislikes unnecessary
clarification should carry over to actions that ask for clarification under
different memory or tool settings, whereas a flat model over complete
combinations would relearn this preference once per combination.

Two further properties of the setting shape our formulation. Onboarding
information is useful but fallible, so it should enter as revisable prior
evidence rather than as a fixed rule. Permissions, tool availability, and
mandatory confirmations are not preferences at all: they determine which
actions may be attempted, and they must be enforced before any exploration. We
therefore formulate personalization of a frozen agent as online learning of a
per-user execution policy over a factorized action space, from scalar feedback
observed only for the selected action, subject to externally specified
feasibility constraints.

\textbf{FABLE} (\emph{Factorized Adaptive Bandit Layer for Execution})
instantiates this formulation as a compact Bayesian policy layer that requires
no access to host weights, gradients, or hidden states and no changes to tool
implementations. Onboarding initializes a revisable residual state. At each
interaction, a context adapter summarizes the request, interaction history,
non-preference background, and hard state; a feasibility filter removes
inadmissible actions; and the policy selects a factorized action over memory
use, information acquisition, and response behavior, which is compiled into
instructions for the host. Expected feedback is decomposed into a prespecified
default score, an operational cost, and a user-specific residual, and Bayesian
Thompson sampling is applied only to the residual, so a single scalar
observation updates the feature directions shared by related actions.

Our contributions are as follows.
\begin{itemize}
\item We formalize per-user execution-policy learning for a frozen,
potentially black-box agent as a feasibility-constrained contextual bandit
over factorized execution actions with selected-action scalar feedback, and
instantiate it as the FABLE algorithm
(Section~\ref{sec:problem-method}).
\item We establish a feasible-oracle regret guarantee for FABLE under a
linear residual-feedback model. For a predictable known default--cost offset
that may vary with context and action, a general Gaussian working
initialization---including the onboarding-informed initialization used by
FABLE---and the calibrated Gaussian sampling schedule
\(\nu_t=\sqrt{9d\log(t/\delta)}\), we adapt linear Thompson-sampling analysis
to obtain \(\widetilde O(d^{3/2}\sqrt n)\) high-probability regret against
the best action in each supplied feasible set, together with an
expected-regret guarantee of the same order
(Theorem~\ref{thm:fable-ts-regret}). Supplementary analyses further
characterize exactly which
preference directions remain identifiable when feasibility constraints
persistently exclude actions
(Supplementary Proposition~2 and Corollary~1), and provide an anytime-valid
confidence-sequence rule that controls false promotion of learned preferences
(Appendix~G).
\item We evaluate FABLE in a four-domain tau2-bench protocol
(Section~\ref{sec:experiments}). The full policy has the highest observed
personalized reward and alignment and ties the highest task-success mean.
Relative to the host, its paired gains are \(+0.077\) reward and \(+0.281\)
alignment, both with positive 95\% CIs; task success remains unresolved.
Matched controls further support cross-action sharing, revisable onboarding,
and online adaptation for preference-sensitive metrics. Appendix~C reports
complementary PAHF and Math500 studies and negative cases.
\end{itemize}

\section{Related Work}
\label{sec:related-work}

\paragraph{Execution control for fixed agents.}
ReAct interleaves language reasoning with environment actions, while Toolformer
trains a language model to invoke external tools
\citep{yao2023react,schick2023toolformer}. These approaches place execution
decisions within the generation process. A complementary line of work learns a
lightweight controller around a fixed agent. \citet{yi2026harness} train an
external controller for a frozen agent harness from offline rollouts. OLIVIA
places a contextual linear bandit at the final action-selection layer of a
frozen ReAct agent and updates it online from action-level feedback, using the
agent's hidden states as contexts~\citep{yu2026olivia}. Concurrent work
MemCon wraps a fixed memory backend with an online contextual-bandit controller
that selects retrieval, plan injection, consolidation, and forgetting
operations from task-level binary feedback~\citep{jiang2026memcon}. OLIVIA
adapts local ReAct action selection, while MemCon controls memory operations
across tasks. FABLE instead maintains a per-user posterior over joint
memory, information-acquisition, and response actions. It requires no access to
host weights, gradients, or hidden states and restricts exploration to an
externally supplied feasible set.

\paragraph{User state and adaptive memory.}
Profiles and long-term memory determine what personal information can be made
available to an agent. LaMP benchmarks profile-conditioned personalization
\citep{salemi2024lamp}, while MemoryBank, MemGPT, A-MEM, and Mem0 develop
mechanisms for storing and retrieving information across interactions
\citep{zhong2023memorybank,packer2023memgpt,xu2025mem,chhikara2025mem0}.
Memory management can itself be adaptive: Reflective Memory Management uses
online reinforcement learning to refine retrieval for long-term personalized
dialogue~\citep{tan2025rmm}. VARS updates per-user retrieval vectors online from
weak scalar feedback~\citep{hao2026vars}, PURPLE uses a contextual bandit to
construct query-specific profiles from user records~\citep{du2026purple}, and
MemToolAgent distills user and environment feedback into memories that guide
later tool use~\citep{er2026memtoolagent}. PersonaAgent connects remembered user
information with downstream actions~\citep{zhang2026personaagent}, whereas
SAGER evolves a user-specific natural-language policy skill for recommendation
reasoning~\citep{tao2026sager}. These methods adapt memory content, retrieval, or
the reasoning policy itself. FABLE holds those mechanisms fixed and treats
memory mode as one component of a broader joint execution action.

\paragraph{Interactive personalization and information acquisition.}
Personalization can occur at test time through explicit preferences or
continuing user interaction. Amulet realigns a frozen LLM from an explicit
preference prompt by treating each token distribution as an online
optimization problem and applying a closed-form proximal-FTRL update
\citep{zhang2025amulet}. T-POP instead learns a neural reward model from
online pairwise preference feedback and combines test-time alignment with
neural dueling-bandit exploration~\citep{qu2025tpop}. PAHF and
MultiSessionCollab study preferences revealed over repeated interactions
\citep{liang2026pahf,mehri2026multisessioncollab}, while PrefDisco studies
just-in-time preference discovery for personalized reasoning
\citep{li2025prefdisco}. Some methods acquire missing preferences or
specifications by questioning the user directly. ADAPT evaluates active
preference elicitation in underspecified long-horizon tasks and introduces
Reflection-DPO to train this behavior~\citep{patel2025adapt}, while
\citet{zhang2025clarify} study when ambiguity and user tolerance warrant a
clarifying question. User preferences also affect tool selection:
ToolSpectrum evaluates tool use conditioned on user profiles and
environmental factors~\citep{cheng2025toolspectrum}, while MPT models latent
preferences for cross-session tool calling~\citep{yoon2026mpt}. Amulet and
T-POP intervene directly in token-level decoding, whereas the remaining
methods specialize in particular interaction channels. FABLE instead learns
an external execution policy over joint memory, information-acquisition, and
response actions from scalar feedback on the executed action. It leaves
generation to a fixed, potentially black-box host and applies externally
specified feasibility constraints before exploration.
\paragraph{Model-level personalization.}
Per-user fine-tuning and personalized alignment adapt model parameters or
decoding behavior to user-specific data or heterogeneous preferences
\citep{tan2024oppu,poddar2024personalizing,chen2024pal,park2024rlhf,chen2024pad}.
Preference Agents use a small local model to generate natural-language rules
that steer a larger fixed model~\citep{shashidhar2024preferenceagents}, while
neural-bandit personalization updates soft instruction embeddings of a
white-box LLM from online feedback~\citep{chen2024neuralbandits}. Behavioral
feedback in FABLE instead updates an external posterior over explicit
execution actions; the compiler and host remain fixed.

\paragraph{Structured contextual bandits and constrained action sets.}
Contextual bandits formalize learning from feedback observed only for the
selected action. LinUCB applies this framework to personalized recommendation
\citep{li2010contextual}, while linear-bandit analyses and linear Thompson
sampling characterize exploration and regret under linear reward models
\citep{abbasi2011improved,agrawal2013thompson,abeille2017linear}. Large-action,
factored, and contextual combinatorial bandits exploit structure in the action
space, including settings with only scalar feedback for the selected joint
action
\citep{zhu2022contextual,zimmert2018factored,zierahn2023nonstochastic}. FABLE
uses factorized context--action features to share information across related
joint execution choices. Influence-diagram bandits represent general
action--latent--observation dependencies and apply structured posterior
sampling~\citep{yu2020graphical}. Mixed-effect Thompson sampling similarly
relates actions through shared effect parameters~\citep{aouali2023mixed}, while
IntelligentPooling uses partial pooling to learn personalized policies when
each user contributes little data~\citep{tomkins2021intelligentpooling}.
FABLE uses a simpler linear representation tied to interpretable
agent-execution components and currently maintains independent posteriors
across users.

Baseline-adjusted bandit models separate an action effect from a flexible,
action-independent baseline
\citep{greenewald2017action,krishnamurthy2018semiparametric}. The offset in
FABLE has a different role: it is a known, action-dependent default--cost score,
and the learned residual represents the target user's departure from that
score. Warm-start contextual bandits combine supervised examples with
subsequent bandit feedback~\citep{zhang2019warmstarting}. In FABLE, soft
onboarding supplies finite-precision prior information rather than supervised
action labels; preference-bearing onboarding is not reused as a per-round
context feature.

Sleeping-bandit models allow the available action set to vary
\citep{kleinberg2010sleeping}. Conservative bandits impose
baseline-performance requirements, whereas safe linear bandits learn under
uncertain safety constraints
\citep{kazerouni2017conservative,amani2019linear,moradipari2020safe}. FABLE
assumes that the surrounding system supplies a predictable, nonempty feasible
set before each decision. The learner neither estimates nor expands this set,
and its comparator is the action with the largest conditional mean under the
personalized reward model within the same context-dependent feasible set.
Accordingly, the theory concerns learning and identifiability under the supplied
constraints, not the validity or estimation of the constraints themselves.

\section{Learning Personalized Execution Policies}
\label{sec:problem-method}

The central difficulty is to learn a \emph{joint} execution policy from sparse
selected-action feedback when the host itself is frozen. Each round reveals one
scalar outcome for a complete execution choice, so the learner must decide
which other choices inherit that evidence. At the same time, it should retain
the host's useful generic behavior, treat stated preferences as revisable
rather than permanent, and never explore an action excluded by external
constraints. These requirements cannot be solved independently: the
representation determines what both onboarding and online feedback mean, and
the feasible set determines where posterior uncertainty may be expressed.

FABLE constructs one bandit policy around this dependency. Sparse joint
feedback is projected onto factorized residual coordinates; onboarding
initializes uncertainty in those same coordinates; posterior sampling turns
the remaining uncertainty into choices only within the supplied feasible set;
and the resulting choices adapt behavior while the compiler and host remain
fixed. Thus each stage consumes the object produced by the preceding stage
rather than contributing a detachable module. Figure~\ref{fig:fable-system-workflow}
shows the resulting recurrent loop.

\begin{figure}[t]
\centering
\includegraphics[width=\columnwidth]{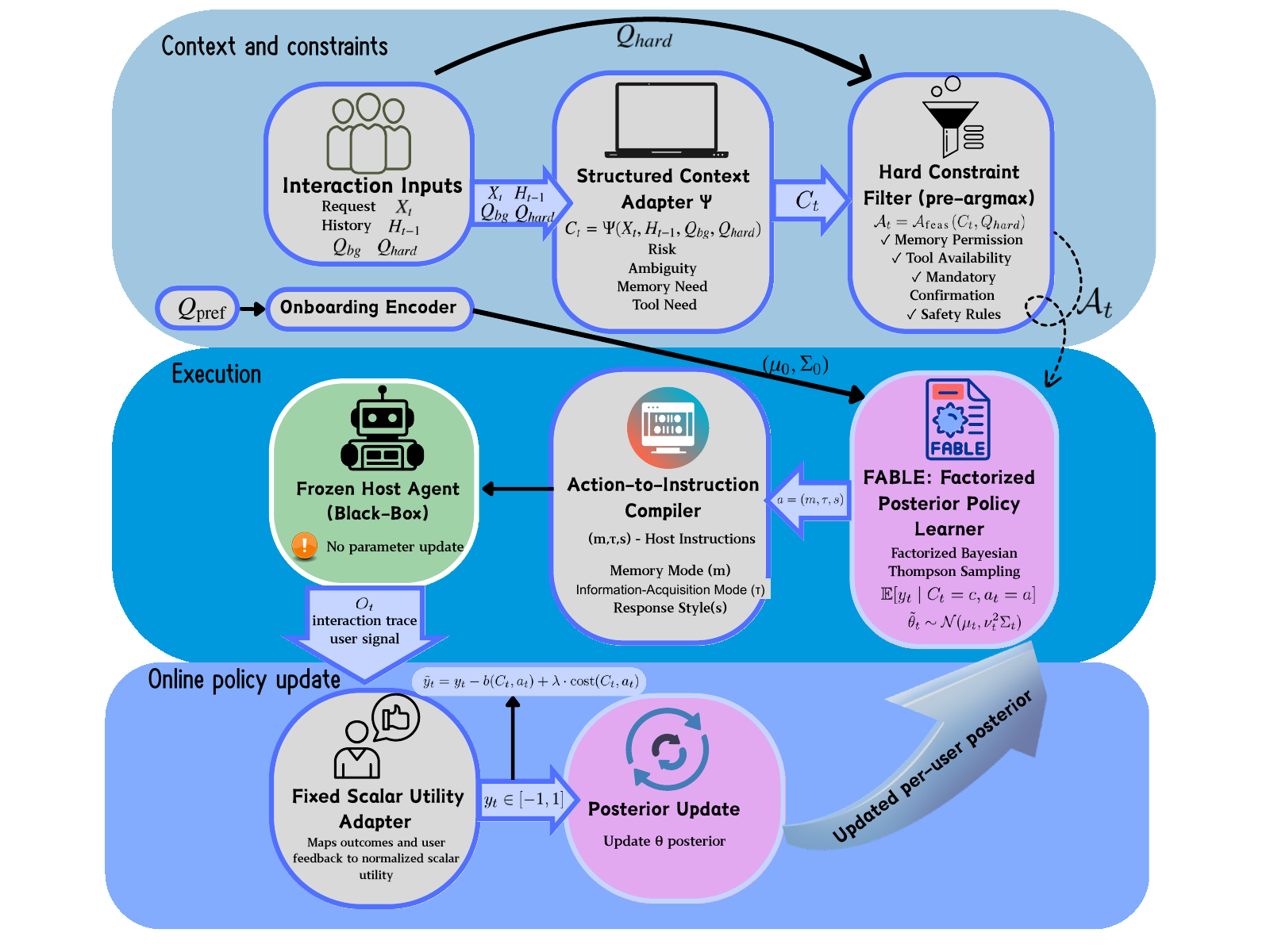}
\caption{FABLE around a frozen host: onboarding initializes the residual-
preference posterior, hard constraints restrict selection to \(\calA_t\),
and residualized feedback updates it online.}
\label{fig:fable-system-workflow}
\end{figure}

\subsection{Sparse Feedback on Joint Execution}

We consider repeated interactions with one target user and suppress the user
index. Because the host may be frozen or black-box, the learned object is an
external execution policy, not a modification of LLM parameters.

An execution decision is the complete joint choice
\(a=(m,\tau,s)\in\calM\times\calT\times\calS\), where \(m\) is the memory
mode, \(\tau\) is the information-acquisition mode, and \(s\) is the response
behavior. The value \(\tau=\text{ask user}\) is an intermediate acquisition
action for obtaining missing information, whereas
\(s=\text{ask clarification}\) is a response or termination protocol.
Appendix~E gives the concrete agent-execution
catalog. Let \(\calA=\calM\times\calT\times\calS\) denote the complete
catalog.

Let \(\mathcal C_{\mathrm{ctx}}\) be the policy-context space, and let
\(c\in\mathcal C_{\mathrm{ctx}}\) denote a generic context used in
\(\phi(c,a)\), \(b(c,a)\), and \(\cost(c,a)\).

Before online interaction, preference-bearing onboarding information
\(Q^{\mathrm{pref}}\) is used only to initialize the residual posterior.
At round \(t\), the policy observes a request \(X_t\), host-readable state
\(H_{t-1}\), optional non-preference background information
\(Q^{\mathrm{bg}}\), and round-specific hard state
\(Q_t^{\mathrm{hard}}\). The hard state records permissions, resource
availability, and mandatory requirements. Keeping \(Q^{\mathrm{pref}}\) out of
the per-round context adapter avoids counting the same stated preference both
as prior evidence and as a recurring context feature.

The fixed context adapter maps the online inputs to
\[
\begin{aligned}
C_t&=(k_t,r_t,g_t,p_t,q_t)\in\mathcal C_{\mathrm{ctx}},\\
C_t&=\Psi(X_t,H_{t-1},Q^{\mathrm{bg}},Q_t^{\mathrm{hard}}),
\end{aligned}
\]
where \(k_t\in\calK\) is the task type and
\(r_t,g_t,p_t,q_t\in[0,1]\) summarize risk, ambiguity, memory need, and
information-acquisition need. These coordinates form a policy-facing summary
rather than a complete representation of the request.

Given \(C_t\), the surrounding system supplies
\(\calA_t=\calA_{\feas}(C_t,Q_t^{\mathrm{hard}})\subseteq\calA\).
Permissions, tool availability, mandatory confirmation, and incompatible
component combinations define feasibility rather than reward.

After the policy selects an action \(a_t\in\calA_t\), the host produces
\(O_t=\mathcal G(X_t,H_{t-1},a_t,\xi_t)\), where \(\xi_t\) represents
execution randomness. The map \(\mathcal G\) combines the
action-to-instruction compiler and the host agent. A feedback adapter maps the
outcome, execution trace, and available user or benchmark signal to scalar
feedback \(y_t\in[-1,1]\). Feedback is observed only for the selected action.
Action semantics and the feedback scale remain stable across rounds. Thus, one
interaction follows
\((X_t,H_{t-1},Q^{\mathrm{bg}},Q_t^{\mathrm{hard}})
\to C_t\to a_t\to O_t\to y_t\). Crucially, \(y_t\) evaluates only the
selected complete action: the learner observes neither separate rewards for
\(m\), \(\tau\), and \(s\), nor outcomes for unselected combinations. A flat
arm model would discard repeated component structure, while three independent
policies would discard complementarities among components. This sparse joint
feedback dictates the representation that follows.

\subsection{Factorized Residual Preference Model}

The model must both transfer one selected-action outcome across related joint
choices and isolate the user's departure from behavior the frozen system
already supplies. Factorization and residualization address these demands in
one preference model.

FABLE uses a fixed feature map \(\phi(c,a)\in\R^d\) with the low-order block
structure
\[
\begin{aligned}
\phi(c,a)=\big[&
\phi_M(m),\ \phi_T(\tau),\ \phi_S(s),\\
&\phi_{C\times M}(c,m),\ \phi_{C\times T}(c,\tau),\
\phi_{C\times S}(c,s),\\
&\phi_{M\times T}(m,\tau),\ \phi_{M\times S}(m,s),\
\phi_{T\times S}(\tau,s)
\big].
\end{aligned}
\]
The learner still chooses one complete action. Main effects share its scalar
outcome across combinations; context--component interactions retain
context-dependent effects; and component--component interactions retain
low-order complementarities. Higher-order interactions would approach a flat
action table and lose this transfer. Identification still depends on variation
among observed feasible actions. Appendix~E
specifies the dictionary.

FABLE models expected feedback as
\begin{equation}
\E[y_t\mid C_t=c,a_t=a]
=
b(c,a)+\phi(c,a)^\top\theta_\star-\lambda\,\cost(c,a),
\label{eq:reward-model}
\end{equation}
where \(b(c,a)\) is a fixed default score, \(\cost(c,a)\ge 0\) is a fixed
operational cost, \(\lambda\ge 0\) is the cost weight, and
\(\theta_\star\in\R^d\) is the target user's true residual preference
parameter. The default and cost are prespecified rather than learned
components: they preserve generic behavior and expose operational trade-offs
for audit. The Gaussian state below learns only the user's departure from that
reference rather than asking sparse per-user feedback to relearn the full
execution utility from scratch.

The implementation uses additive default and cost components for memory,
information acquisition, ambiguity, risk, task type, and response behavior;
Appendix~F gives their concrete specification. Oracle
comparisons use the best action under Equation~\eqref{eq:reward-model} in the
same feasible set \(\calA_t\). Define the known reference score and its
round-specific form by
\[
\bar b(c,a)=b(c,a)-\lambda\cost(c,a),
\qquad
\bar b_t(a)=\bar b(C_t,a),
\]
and write \(x_t(a)=\phi(C_t,a)\). Thus every selected-action outcome is first
centered against the frozen system's reference behavior and then attributed to
feature directions shared by related joint choices. The next problem is cold
start: any onboarding signal must initialize these very directions, or
onboarding and online feedback would describe different preference objects.

\subsection{Revisable Onboarding in the Same Coordinates}

The residual model makes later feedback reusable but does not remove the
cold-start problem. Because stated preferences may be wrong, hard onboarding
would prevent adaptation. FABLE instead treats it as finite-precision evidence
about the Gaussian working-posterior variable \(\theta\) whose coordinates are
defined by the factorized residual model.

Let \(J\in\mathbb N_0\) be the number of retained onboarding
pseudo-observations. An onboarding encoder maps \(Q^{\mathrm{pref}}\) to
\(J\) tuples
\((v_j,u_j,\kappa_j)\), where \(v_j\in\R^d\) is a preference direction,
\(u_j\in[-1,1]\) is its signed response, and \(\kappa_j\ge 0\) is the assigned
precision. Each tuple represents the pseudo-observation
\(u_j=v_j^\top\theta+\epsilon_j\), with
\(\epsilon_j\sim\Normal(0,\kappa_j^{-1})\). A tuple with \(\kappa_j=0\) is
omitted. The no-onboarding case is \(J=0\); the tuple collection and the sums
below are then empty.

Starting from the base prior
\(\theta\sim\Normal(\mu_{\mathrm{base}},\Sigma_{\mathrm{base}})\), where
\(\Sigma_{\mathrm{base}}\succ0\), the initial Gaussian state is
\begin{align}
\Sigma_0^{-1}
&=
\Sigma_{\mathrm{base}}^{-1}
+
\sum_{j=1}^J \kappa_jv_jv_j^\top,
\label{eq:init-precision}\\
\mu_0
&=
\Sigma_0
\left(
\Sigma_{\mathrm{base}}^{-1}\mu_{\mathrm{base}}
+
\sum_{j=1}^J \kappa_jv_ju_j
\right).
\label{eq:init-mean}
\end{align}
When \(J=0\), this construction recovers the base prior. Because the
onboarding precision is finite and the tuples use the online residual
coordinates, behavioral feedback can revise the cold-start bias without a
coordinate translation. Any fixed elicitation procedure that supplies
\((\mu_0,\Sigma_0)\) in the same feature space may replace the
pseudo-observation construction. The initial state is thus the first state of
the same learner that receives selected-action residual feedback, not a
separate preference store. Because its precision is finite, the initialized
state deliberately retains uncertainty; the policy must decide how to resolve
that uncertainty through feasible online behavior.

\subsection{Posterior Sampling Within the Feasible Set}

Finite onboarding precision leaves uncertainty for behavior to resolve, but
hard constraints determine where that uncertainty may be explored. The
surrounding system therefore supplies the nonempty set
\(\calA_t=\calA_{\feas}(C_t,Q_t^{\mathrm{hard}})\) before FABLE scores an
action. Encoding a prohibition as a large cost would not preserve this order:
an optimistic posterior draw could still select the prohibited action.
Feasibility restricts the decision domain, whereas the default and cost remain
known terms inside the score.

Let \(\mathcal F_t\) denote the sigma-field available immediately before the
round-\(t\) Thompson sample is drawn. It contains observations from rounds
\(1,\ldots,t-1\), the current context \(C_t\), the supplied set \(\calA_t\),
and the known quantities
\[
\{\phi(C_t,a),\bar b(C_t,a):a\in\calA_t\}.
\]
For \(t\geq1\), FABLE maintains the Gaussian working state
\(\theta\mid\mathcal F_t\sim\Normal(\mu_t,\Sigma_t)\), with
\(\Lambda_t=\Sigma_t^{-1}\) and \(h_t=\Lambda_t\mu_t\). The onboarding state
initializes the first decision through
\(\Lambda_1=\Sigma_0^{-1}\), \(h_1=\Lambda_1\mu_0\), and hence
\((\mu_1,\Sigma_1)=(\mu_0,\Sigma_0)\).

FABLE samples only the uncertain residual parameter,
\[
\widetilde\theta_t
\sim
\Normal(\mu_t,\nu_t^2\Sigma_t),
\]
where \(\nu_t>0\) is a specified sampling scale, and selects
\begin{equation}
a_t
=
\operatorname*{arg\,max}_{a\in\calA_t}
\left[
\bar b(C_t,a)+\phi(C_t,a)^\top\widetilde\theta_t
\right].
\label{eq:ts-score}
\end{equation}
Ties are resolved by taking the earliest action in a predeclared fixed catalog
order. The default and cost are not sampled because they are prespecified. The
maximization is restricted to \(\calA_t\), so the selected action belongs to
the supplied feasible set. Let
\[
\mathcal F_t^{\mathrm{act}}
:=
\sigma(\mathcal F_t,\widetilde\theta_t,a_t)
\]
denote the information available after sampling and action selection but before
round-\(t\) feedback is observed.

After execution, FABLE forms the residual feedback
\(\widetilde y_t=y_t-\bar b(C_t,a_t)\). Let
\(\phi_t=\phi(C_t,a_t)\). Under a Gaussian working likelihood
\(\widetilde y_t=\phi_t^\top\theta+\varepsilon_t\), with working noise scale
\(\sigma^2\), the precision and information vector are updated as
\begin{align}
\Lambda_{t+1}
&=
\Lambda_t+\sigma^{-2}\phi_t\phi_t^\top,
\label{eq:precision-update}\\
h_{t+1}
&=
h_t+\sigma^{-2}\phi_t\widetilde y_t.
\label{eq:info-update}
\end{align}
The next Gaussian state is recovered through
\(\Sigma_{t+1}=\Lambda_{t+1}^{-1}\) and
\(\mu_{t+1}=\Sigma_{t+1}h_{t+1}\). Under conditional Gaussian noise, this is
a conjugate posterior update. Under bounded or conditionally sub-Gaussian
feedback, it is the Gaussian working posterior used by FABLE. Crucially, the
same factorized residual coordinates that received finite-precision onboarding
now receive selected-action evidence. Theorem~\ref{thm:fable-ts-regret}
specifies a calibrated choice of \(\nu_t\) for the feasible-oracle regret
guarantee.

\subsection{Adaptive Policy Around a Frozen Agent}

Algorithm~\ref{alg:fable} places this statistical loop around the
compiler-and-host map \(\mathcal G\). Onboarding and online feedback update one
residual state, which affects the next joint choice only after feasibility
filtering; the host realizes that choice without changing its parameters or
execution semantics.

\begin{algorithm}[!b]
\caption{FABLE: Factorized Adaptive Bandit Layer for Execution}
\label{alg:fable}
\small
\begin{algorithmic}[1]
\REQUIRE Action spaces \(\calM,\calT,\calS\); context adapter \(\Psi\);
feasibility filter \(\calA_{\feas}\); feature map \(\phi\); default score \(b\);
cost \(\cost\); \(\lambda\); sampling schedule \((\nu_t)_{t\ge1}\);
\(\sigma^2\); base prior; onboarding and feedback adapters; compiler-and-host
map \(\mathcal G\).
\STATE Observe \(Q^{\mathrm{pref}}\), encode
\(\{(v_j,u_j,\kappa_j)\}_{j=1}^J\), and initialize
\((\mu_0,\Sigma_0)\) by
Equations~\eqref{eq:init-precision}--\eqref{eq:init-mean}.
\STATE Set \(\Lambda_1=\Sigma_0^{-1}\) and \(h_1=\Lambda_1\mu_0\).
\FOR{round \(t=1,2,\ldots\)}
    \STATE Observe \(X_t\), \(H_{t-1}\), \(Q^{\mathrm{bg}}\), and
    \(Q_t^{\mathrm{hard}}\).
    \STATE Set
    \(C_t=\Psi(X_t,H_{t-1},Q^{\mathrm{bg}},Q_t^{\mathrm{hard}})\) and
    \(\calA_t=\calA_{\feas}(C_t,Q_t^{\mathrm{hard}})\).
    \STATE Recover \((\mu_t,\Sigma_t)\), sample
    \(\widetilde\theta_t\sim\Normal(\mu_t,\nu_t^2\Sigma_t)\), and select
    \(a_t\) by Equation~\eqref{eq:ts-score}.
    \STATE Execute
    \(O_t=\mathcal G(X_t,H_{t-1},a_t,\xi_t)\), obtain \(y_t\), and form
    \(\widetilde y_t\).
    \STATE Set \(\phi_t=\phi(C_t,a_t)\) and update
    \((\Lambda_{t+1},h_{t+1})\) by
    Equations~\eqref{eq:precision-update}--\eqref{eq:info-update}.
    \STATE Update \(H_t\) with \((X_t,a_t,O_t,y_t)\).
\ENDFOR
\end{algorithmic}
\end{algorithm}
\FloatBarrier

\paragraph{Fixed semantics, adaptive policy state.}

The context adapter, action semantics, feature map, default--cost offset,
feasible-set rule, compiler-and-host interface, and feedback adapter remain
fixed during online learning, as do the host-agent parameters. FABLE updates
only the per-user statistical state and the evolving interaction history.
The feedback scale and catalog order are fixed as well. Otherwise, an old
observation would no longer correspond to the same residual coordinate,
likelihood, action meaning, or comparator as a new observation.

The posterior already affects current action selection. Writing a learned
preference into persistent host-readable state has a different risk because it
can alter future context construction. FABLE therefore treats promotion as an
optional confidence-controlled writeback of a prespecified, identifiable
contrast. Promotion adds no reward observation and does not modify
Equations~\eqref{eq:precision-update}--\eqref{eq:info-update}; its full
anytime-valid error guarantee remains in Appendix~G. A contrast is written
back only when its anytime confidence sequence excludes zero; otherwise no
promotion occurs.

Taken together, this is one policy rather than a sum of techniques. Sparse
joint feedback requires shared factorized coordinates; retaining generic
behavior requires learning only a residual in those coordinates; fallible
onboarding remains revisable by initializing the same state with finite
precision; and posterior sampling converts the remaining uncertainty into
actions only after feasibility has restricted the domain. Because only the
posterior and interaction history evolve, the resulting behavior adapts around
the frozen host. Each element resolves a necessary consequence of the original
learning problem.
\section{Theoretical Results}
\label{sec:theory}

For \(a\in\calA_t\), let
\[
\begin{aligned}
x_t(a)&=\phi(C_t,a),&
\bar b_t(a)&=\bar b(C_t,a),\\
f_t(a)&=\bar b_t(a)+x_t(a)^\top\theta_\star .
\end{aligned}
\]
The comparator maximizes \(f_t\) over the same supplied set \(\calA_t\).
If the dictionary is redundant, \(\theta_\star\) denotes the score-equivalent
representative minimizing \(\|\theta-\mu_0\|_{\Lambda_0}\); this fixes the
radius below without changing any action score.

\subsection{Assumptions}

\begin{assumption}[Residual linear feedback]
\label{ass:residual-feedback}
For every round and feasible action,
\[
\begin{gathered}
y_t=\bar b_t(a_t)+x_t(a_t)^\top\theta_\star+\varepsilon_t,\qquad
f_t(a)\in[-1,1],\\
\E[\varepsilon_t\mid\mathcal F_t^{\mathrm{act}}]=0,\\
\E[e^{u\varepsilon_t}\mid\mathcal F_t^{\mathrm{act}}]
\le e^{u^2\sigma^2/2}\quad(\forall u\in\R).
\end{gathered}
\]
We normalize \(\sigma^2=1\) for regret. The general-scale posterior recursion
uses \(x_t(a_t)/\sigma\); exact conjugacy additionally uses the working model
\(\varepsilon_t\mid\mathcal F_t^{\mathrm{act}}\sim\Normal(0,\sigma^2)\).
\end{assumption}

\begin{assumption}[Bounded features and prior-centered residual radius]
\label{ass:bounded-residual}
For constants \(L_x,\lambda_0,R_b>0\),
\[
\begin{gathered}
\|x_t(a)\|_2\le L_x\quad(\forall t,a\in\calA_t),\\
\Lambda_0\succeq\lambda_0 I_d,\qquad
\|\theta_\star-\mu_0\|_{\Lambda_0}\le R_b .
\end{gathered}
\]
Here \(\Lambda_0=\Sigma_0^{-1}\) and
\(\|v\|_{\Lambda_0}=\sqrt{v^\top\Lambda_0v}\).
\end{assumption}

\begin{assumption}[Predictable nonempty feasible sets]
\label{ass:predictable-feasible-set}
For every \(t\),
\(\varnothing\ne\calA_t=\calA_{\feas}(C_t,Q_t^{\mathrm{hard}})\), and
\(\calA_t\) is measurable before \(a_t\) is selected.
\end{assumption}

\subsection{Regret of calibrated Thompson-style exploration}

With deterministic tie-breaking, define
\[
\displaystyle
a_t^\star\in\argmax_{a\in\calA_t}f_t(a),\qquad
R_n=\sum_{t=1}^n[f_t(a_t^\star)-f_t(a_t)] .
\]
This realized-context comparator uses the context and feasible-set sequence
generated along FABLE's trajectory, not a counterfactual trajectory.

\begin{theorem}[Feasible-oracle regret of FABLE]
\label{thm:fable-ts-regret}
Suppose Assumptions~\ref{ass:residual-feedback}--%
\ref{ass:predictable-feasible-set} hold.
Initialize the algorithmic state by
\((\mu_1,\Sigma_1)=(\mu_0,\Sigma_0)\) and
\(\Lambda_1=\Sigma_1^{-1}=\Lambda_0:=\Sigma_0^{-1}
\succeq\lambda_0I_d\).
At round \(t\), draw
\(\widetilde\theta_t\sim\Normal(\mu_t,\nu_t^2\Sigma_t)\), select by
Equation~\eqref{eq:ts-score}, and update by
Equations~\eqref{eq:precision-update}--\eqref{eq:info-update}.
For \(n\ge2\), \(\delta\in(0,1/2]\), use
\[
\nu_t=\sqrt{9d\log(t/\delta)} .
\]
Then, with probability at least \(1-\delta\),
\[
\displaystyle R_n=\widetilde O(d^{3/2}\sqrt n).
\]
If \(n\) is known and \(\delta=n^{-2}\), then
\[
\displaystyle \E[R_n]=\widetilde O(d^{3/2}\sqrt n).
\]
The notation suppresses logarithmic factors in \(n,d,1/\delta\) and fixed
constants determined by the feature bound, prior-centered radius, and
Gaussian initialization.
\end{theorem}

The proof is deferred to Appendix~A.

\section{Experiments}
\label{sec:experiments}

The main-text evaluation uses tau2-bench because it tests the complete policy
layer before a frozen tool-using agent with executable domain tools and a
native task-success metric. It asks whether the integrated policy changes the
declared personalized objective, whether matched controls support the roles of
factorization, onboarding, and online updating, and whether these changes
preserve end-to-end task performance. Appendix~C reports the complementary
PAHF and Math500 studies and additional paired
tau2-bench analysis.

\subsection{tau2-bench: Executable Tool-Use Evaluation}
\label{subsec:experiments-tau2}

We test FABLE in tau2-bench~\citep{barres2025tau2}, an executable
customer-service benchmark with native tools and mutable state. Each task is
one bandit round. For every seed--domain shard across Airline, Retail, Telecom,
and Banking Knowledge, 10 tasks are excluded for calibration, followed by 20
online-learning and 20 frozen-evaluation tasks. Formal evaluation contains 80
unique domain--task--profile clusters and 240 episodes per policy.

\begin{table*}[!t]
\centering
\captionsetup{skip=3pt}
\scriptsize
\setlength{\tabcolsep}{4.5pt}
\begin{tabular}{lccc}
\toprule
Comparator
& \(\Delta\) personalized reward [95\% CI]
& \(\Delta\) alignment [95\% CI]
& \(\Delta\) task success [95\% CI] \\
\midrule
Host baseline
& \(+0.077\ \mathbf{[0.027,\,0.127]}\)
& \(+0.281\ \mathbf{[0.210,\,0.354]}\)
& \(+0.017\ [-0.046,\,0.079]\) \\
Flat complete-action LinTS
& \(+0.009\ [-0.026,\,0.045]\)
& \(+0.022\ \mathbf{[0.004,\,0.038]}\)
& \(+0.004\ [-0.046,\,0.054]\) \\
FABLE (frozen)
& \(+0.005\ [-0.041,\,0.051]\)
& \(+0.024\ \mathbf{[0.003,\,0.045]}\)
& \(-0.000\ [-0.063,\,0.058]\) \\
FABLE (no onboarding)
& \(+0.071\ \mathbf{[0.026,\,0.116]}\)
& \(+0.163\ \mathbf{[0.128,\,0.199]}\)
& \(+0.042\ [-0.017,\,0.100]\) \\
FABLE (no promotion)
& \(+0.026\ [-0.015,\,0.066]\)
& \(+0.031\ \mathbf{[0.012,\,0.048]}\)
& \(+0.025\ [-0.029,\,0.079]\) \\
FABLE (no cost)
& \(+0.036\ [-0.002,\,0.076]\)
& \(+0.015\ [-0.000,\,0.030]\)
& \(+0.042\ [-0.008,\,0.096]\) \\
\bottomrule
\end{tabular}
\caption{Paired frozen-evaluation differences (FABLE full minus comparator).
The 95\% cluster-bootstrap CIs use 10,000 resamples of 80
domain--task--profile clusters after seed averaging; bold excludes zero.
Pairing is by seed, domain, task ID, profile, and phase.}
\label{tab:experiments-tau2-v4-paired}
\end{table*}

Efficient and guided profiles prefer concise outcome-first responses and
explanation without excessive verbosity, respectively. Balanced assignment
yields 24 domain--profile--seed trajectories. FABLE sees only the public domain
and assigned profile, never hidden simulator instructions, evaluator criteria,
reference actions, target state, or model output. The action space is
\[
\begin{aligned}
\mathcal{M}&=\{\text{current turn},\text{conversation}\},\\
\mathcal{T}&=\{\text{standard},\text{verify},\text{ask if missing}\},\\
\mathcal{S}&=\{\text{direct},\text{concise},\text{guided}\}.
\end{aligned}
\]
Thus \(a_t\in\mathcal M\times\mathcal T\times\mathcal S\) ranges over 18
complete actions represented by main effects and pairwise interactions. The
selected action becomes a fixed prompt suffix; tau2's tools, environment,
customer, evaluator, and parser remain unchanged.

\paragraph{Comparators.}
The host omits the policy layer, and Rule-Only uses only the fixed
default-and-cost score. Global linear Thompson sampling (LinTS) shares a
factorized state across profiles within each domain; Per-User LinTS uses an
uninformed one per domain--profile pair. Flat complete-action LinTS matches
full FABLE except for assigning one coordinate per complete action. FABLE
(frozen) never updates or promotes; the other ablations remove onboarding,
promotion, or training cost.

\paragraph{Metrics.}
Let \(R_t\) be tau2's native reward, \(A_t\) the prespecified deterministic
verbosity-alignment score for the independently assigned profile, and
\(c(a_t)\) the prespecified action cost. The common evaluation signal and
reportable personalized reward are
\[
\begin{aligned}
y_t^{\mathrm{eval}}
&=\operatorname{clip}_{[-1,1]}
\left(2[0.75R_t+0.25A_t]-1-0.2c(a_t)\right),\\
U_t&=(y_t^{\mathrm{eval}}+1)/2.
\end{aligned}
\]
All policies use the same evaluation formula and cost weight; the no-cost arm
sets it to zero only during training. We report \(U_t\), \(A_t\), and binary
native task success separately. \(U_t\) is the declared joint objective, while
\(A_t\) measures one controlled synthetic verbosity preference rather than
general preference alignment.

\begin{table}[t]
\centering
\captionsetup{skip=3pt}
\scriptsize
\setlength{\tabcolsep}{1.5pt}
\begin{tabular}{@{}lccc@{}}
\toprule
Policy
& Pers. reward \(\uparrow\)
& Align. \(\uparrow\)
& Success \(\uparrow\) \\
\midrule
Host baseline          & \(0.563 \pm 0.021\) & \(0.435 \pm 0.003\) & \(0.608 \pm 0.029\) \\
Rule-Only              & \(0.558 \pm 0.022\) & \(0.463 \pm 0.023\) & \(0.596 \pm 0.036\) \\
Global LinTS           & \(0.552 \pm 0.043\) & \(0.481 \pm 0.018\) & \(0.588 \pm 0.057\) \\
Per-User LinTS         & \(0.580 \pm 0.014\) & \(0.579 \pm 0.072\) & \(0.592 \pm 0.040\) \\
Flat complete-action LinTS
                       & \(0.631 \pm 0.023\) & \(0.695 \pm 0.016\) & \(0.621 \pm 0.026\) \\
FABLE (no onboarding) & \(0.569 \pm 0.020\) & \(0.553 \pm 0.060\) & \(0.583 \pm 0.007\) \\
FABLE (no promotion)  & \(0.614 \pm 0.010\) & \(0.686 \pm 0.004\) & \(0.600 \pm 0.013\) \\
FABLE (no cost)       & \(0.604 \pm 0.041\) & \(0.702 \pm 0.038\) & \(0.583 \pm 0.047\) \\
FABLE (frozen)        & \(0.635 \pm 0.037\) & \(0.693 \pm 0.014\) & \(\mathbf{0.625 \pm 0.045}\) \\
FABLE (full)          & \(\mathbf{0.640 \pm 0.036}\)
                       & \(\mathbf{0.717 \pm 0.013}\)
                       & \(\mathbf{0.625 \pm 0.043}\) \\
\bottomrule
\end{tabular}
\caption{Tau2 frozen evaluation (240 episodes/policy): mean \(\pm\) SD over
three seed means. Bold marks the highest observed mean, including ties.}
\label{tab:experiments-tau2-v4-main}
\end{table}

FABLE (full) is the only policy at the top of all three columns in
Table~\ref{tab:experiments-tau2-v4-main}: it has the highest observed
personalized reward and alignment and ties FABLE (frozen) for the highest
observed task success. Relative to the host, personalized reward increases by
\(0.077\) (95\% CI \([0.027,0.127]\)) and alignment by \(0.281\)
(\([0.210,0.354]\)). Task success increases by \(0.017\)
(\([-0.046,0.079]\)), but this interval includes zero; we therefore do not
claim a statistically significant task-success improvement.

Matched contrasts support specific links in the policy rather than uniform
superiority. Relative to flat complete-action LinTS, factorization improves
alignment by \(0.022\) (95\% CI \([0.004,0.038]\)); reward and task success
remain unresolved. Relative to FABLE (frozen), online adaptation improves
alignment by \(0.024\) (\([0.003,0.045]\)) with unchanged aggregate task
success. Removing onboarding produces the largest measured loss: full minus
no-onboarding is \(+0.163\) alignment (\([0.128,0.199]\)) and \(+0.071\)
personalized reward (\([0.026,0.116]\)).
Table~\ref{tab:experiments-tau2-v4-paired} reports the prespecified matched
contrasts used for component attribution;
Appendix~C gives the corresponding detailed analysis.

Task success is lower than the host on Airline, tied on Retail, and higher on
Telecom and Banking Knowledge. Host noninferiority is not established because
the lower confidence endpoint falls below the prespecified \(-0.02\) margin.
Thus the experiment supports controlled synthetic verbosity adaptation and
executable integration, but not significant task-success improvement, general
human preference alignment, or uniform domain-level gains.

\section{Conclusion}

FABLE forms one constrained policy around a frozen host: factorized residuals
share sparse joint feedback, revisable onboarding initializes the same
coordinates, and posterior sampling follows feasibility filtering.
Theorem~\ref{thm:fable-ts-regret} bounds feasible-oracle regret. Tau2 matched
controls support these links for preference-sensitive metrics, but neither
task-success improvement nor uniform domain gains; these claims require fixed
semantics, stable feedback, and valid constraints.

\clearpage

\appendix

\input{appendix/fable_theorem1_proof}
\input{appendix/additional_theory}
\input{appendix/additional_experiments}


\section{Optional Coordinate-Sparse Subfactorization}
\label{subsec:pure-spca-subfactorization}

The main analysis uses the full factorized context--action feature map. When
the initial interaction budget is small, an optional cold-start variant can
instead restrict learning to a subset of its existing coordinates. Let
\[
    \phi_{\mathrm{fac}}(C,a)\in\R^{d_f}
\]
denote the current factorized feature map, such as the feature map in
Equation~\eqref{eq:factorized-feature}.  Its coordinates already have semantic
meanings: memory-mode effects, tool-mode effects, answer-style effects,
task-style interactions, task-tool interactions, need interactions, and
action-component interactions. Full FABLE retains all coordinates; the optional
variant, denoted by \textbf{FABLE-SPCA}, chooses a
binary mask
\[
    z\in\{0,1\}^{d_f},
    \qquad
    d_z=\mathbf 1^\top z<d_f,
\]
and runs the bandit only on the selected coordinates.

Throughout this appendix, let
\[
Q_{\mathrm{init}}
=
\bigl(Q^{\mathrm{pref}},Q_{\mathrm{persist}}^{\mathrm{hard}}\bigr)
\]
denote the information available when the mask is chosen, where
\(Q_{\mathrm{persist}}^{\mathrm{hard}}\) contains only constraints assumed to
remain fixed over the online horizon. For brevity, we write \(Q\) for
\(Q_{\mathrm{init}}\) below. Its preference component determines the onboarding
posterior, while its persistent hard component may constrain the mask. The
round-specific state \(Q_t^{\mathrm{hard}}\) continues to enter the online feasibility
filter but does not change the fixed mask.

The coordinate-preserving subfactorization is
\[
    \phi_z(C,a)=S_z^\top\phi_{\mathrm{fac}}(C,a),
\]
where \(S_z\) is a coordinate selection matrix. The selection objective
preserves as much onboarding-weighted personalized prediction variance as
possible. The bandit estimation term is not part of this coordinate-sparse PCA
objective; it enters the regret and sample-efficiency analysis below.\paragraph{Parameterization convention.}
The implementation feature dictionary may contain linearly dependent columns.
Let
\[
    \mathcal S_{\mathrm{dict}}
    =
    \operatorname{span}
    \left\{
        \phi(c,a):(c,a)\text{ is admissible}
    \right\}.
\]
Here and below, admissibility is with respect to the fixed dictionary.
Two parameters \(\theta\) and \(\theta'\) are \emph{score-equivalent} if
\[
    \phi(c,a)^\top\theta
    =
    \phi(c,a)^\top\theta',
    \qquad
    \forall (c,a)\text{ admissible}.
\]
When the true residual score has more than one parameter representation, we use
its unique prior-centered representative
\[
    \theta_\star
    \in
    \operatorname*{arg\,min}_{\theta\in\Theta_\star}
    \|\theta-\mu_0\|_{\Lambda_0},
\]
where \(\Theta_\star\) is the affine set of score-equivalent true parameters.
Uniqueness follows from \(\Lambda_0\succ0\). Coefficient-level preference
functionals are used only for directions in \(\mathcal S_{\mathrm{dict}}\), so
their values are invariant to score-equivalent reparameterizations. If the
feature dictionary is nonredundant, this convention has no effect.

Under the residual model below, the conditional mean reward is
\[
    f_t(a)
    =
    \bar b_t(a)+x_t(a)^\top\theta_\star.
\]
Choose a deterministic tie-breaking rule and let
\[
    a_t^\star
    \in
    \operatorname*{arg\,max}_{a\in\calA_t} f_t(a).
\]
For horizon \(n\), we use the feasible-oracle regret \(R_n\) defined in
Section 4 of the main paper.
The oracle is restricted to the same supplied feasible set \(\calA_t\) as the
algorithm.

\paragraph{Onboarding posterior in the existing factorized coordinates.}
Because the FABLE coordinate dictionary is already fixed, the onboarding prior
is constructed exactly as in Equations (2)--(3) of the main paper.
The LLM parser maps onboarding text and user-provided initialization information
into sparse semantic preference directions
\[
\begin{aligned}
    v_j&\in\R^{d_f},
    &u_j&\in[-1,1],\\
    \kappa_j&\ge 0,
    &j&=1,\ldots,J.
\end{aligned}
\]
For example,
\[
    v_j=e_{\mathrm{style:concise}}-e_{\mathrm{style:detailed}}
\]
encodes a preference for concise over detailed answers.  A coding-specific style
preference may use an interaction direction such as
\[
    v_j=
    e_{\substack{\mathrm{task:coding}\\
                 {}\times\mathrm{style:step\mbox{-}by\mbox{-}step}}}
    -
    e_{\substack{\mathrm{task:coding}\\
                 {}\times\mathrm{style:direct}}}.
\]

With base prior
\[
    \theta\sim\Normal(\mu_{\mathrm{base}},\Sigma_{\mathrm{base}}),
    \qquad
    \Sigma_{\mathrm{base}}^{-1}=\lambda_0I_{d_f},
\]
and Gaussian pseudo-observations
\[
    u_j=v_j^\top\theta+\epsilon_j,
    \qquad
    \epsilon_j\sim\Normal(0,\kappa_j^{-1}),
\]
the full-factorized onboarding posterior is
\[
    \theta\mid Q\sim\Normal(\mu_Q,\Sigma_Q),
\]
where
\begin{equation}
\label{eq:pure-spca-prior-precision}
    \Sigma_Q^{-1}
    =
    \Sigma_{\mathrm{base}}^{-1}
    +
    \sum_{j=1}^J\kappa_jv_jv_j^\top,
\end{equation}
and
\begin{equation}
\label{eq:pure-spca-prior-mean}
    \mu_Q
    =
    \Sigma_Q
    \left(
        \Sigma_{\mathrm{base}}^{-1}\mu_{\mathrm{base}}
        +
        \sum_{j=1}^J\kappa_jv_ju_j
    \right).
\end{equation}
Define the posterior second moment
\begin{equation}
\label{eq:pure-spca-second-moment}
    M_Q=\Sigma_Q+\mu_Q\mu_Q^\top.
\end{equation}
This matrix measures which already-factorized FABLE coordinates are likely to be
important for this user after onboarding.

\paragraph{LLM-estimated early-context distribution.}
The subfactorization should preserve coordinates that are both user-relevant and
likely to be activated in the user's near-term requests.  We use an
onboarding-conditioned early-context distribution \(\nu_Q\).  In practice, an LLM
constructs a small set of structured context prototypes
\[
\begin{aligned}
    \widehat C^{(\ell)}
    &=
    (\widehat k_\ell,\widehat r_\ell,\widehat g_\ell,
     \widehat p_\ell,\widehat q_\ell),\\
    \omega_\ell&\ge 0,
    &\sum_{\ell=1}^{L_Q}\omega_\ell&=1,
\end{aligned}
\]
from the user's textual prior and current query.  These are the same structured
context variables used by the existing FABLE feature map.  The LLM does not
choose the mask directly; it only estimates structured context scores. Given a
prototype \(\widehat C^{(\ell)}\), the algorithm forms the feasible action set
\[
    \widehat{\mathcal A}^{(\ell)}
    =
    \mathcal A_{\feas}
    (\widehat C^{(\ell)},Q_{\mathrm{persist}}^{\mathrm{hard}})
\]
and a default action distribution, for example
\[
    \pi_0(a\mid \widehat C^{(\ell)})
    =
    \frac{\exp\{\beta_{\mathrm{temp}}\bar b_\ell(a)\}}
    {\sum_{a'\in\widehat{\mathcal A}^{(\ell)}}
     \exp\{\beta_{\mathrm{temp}}\bar b_\ell(a')\}}.
\]
Here \(\beta_{\mathrm{temp}}\ge0\) is a prespecified inverse temperature, and
\[
\begin{aligned}
    \bar b_\ell(a)
    &=
    b(\widehat C^{(\ell)},a)\\
    &\quad-\lambda\cost(\widehat C^{(\ell)},a).
\end{aligned}
\]
Write
\(\phi_\ell(a)=\phi_{\mathrm{fac}}(\widehat C^{(\ell)},a)\).
The empirical feature covariance is
\begin{equation}
\label{eq:pure-spca-Ghat}
    \widehat G_Q
    =
    \sum_{\ell=1}^{L_Q}\omega_\ell
    \sum_{a\in\widehat{\mathcal A}^{(\ell)}}
    \pi_0(a\mid \widehat C^{(\ell)})
    \phi_\ell(a)\phi_\ell(a)^\top.
\end{equation}
For the population theory below, write
\begin{equation}
\label{eq:pure-spca-GQ}
    G_Q
    =
    \E_{C\sim\nu_Q,\ a\sim\pi_0(\cdot\mid C)}
    \left[
        \phi_{\mathrm{fac}}(C,a)\phi_{\mathrm{fac}}(C,a)^\top
    \right].
\end{equation}

\paragraph{Selection matrix and coordinate-preserving subfactorization.}
For a binary mask \(z\in\{0,1\}^{d_f}\), let
\[
    \mathcal I(z)=\{i:z_i=1\},
    \qquad
    d_z=|\mathcal I(z)|.
\]
Let
\[
    S_z=[e_i:i\in\mathcal I(z)]\in\{0,1\}^{d_f\times d_z}
\]
be the selection matrix and let
\begin{equation}
\label{eq:pure-spca-mask}
    Z=S_zS_z^\top=\operatorname{diag}(z)
\end{equation}
be the diagonal coordinate mask.  The active feature vector and active user vector
are
\begin{equation}
\label{eq:pure-spca-feature-user}
    \phi_z(C,a)=S_z^\top\phi_{\mathrm{fac}}(C,a)\in\R^{d_z},
    \qquad
    \theta_{\star,z}=S_z^\top\theta_\star\in\R^{d_z}.
\end{equation}
Thus FABLE-SPCA compresses from the existing FABLE dimension \(d_f\) to a smaller
semantic dimension \(d_z\), but it never rotates or mixes coordinates.

\begin{lemma}[Subfactorization preserves semantic coordinates]
\label{lem:pure-spca-coordinate-preservation}
For any binary mask \(z\), the active feature vector \(\phi_z(C,a)\) and the
active true user vector \(\theta_{\star,z}\) have the same dimension. Their \(r\)-th
coordinates correspond to the same coordinate of the original FABLE feature map.
Moreover,
\[
    \phi_z(C,a)^\top\theta_{\star,z}
    =
    \phi_{\mathrm{fac}}(C,a)^\top Z\theta_\star.
\]
\end{lemma}

\begin{proof}
Write \(\mathcal I(z)=\{i_1,\ldots,i_{d_z}\}\).  The \(r\)-th column of \(S_z\)
is \(e_{i_r}\).  Hence
\[
    \phi_z(C,a)[r]=\phi_{\mathrm{fac}}(C,a)[i_r],
    \qquad
    \theta_{\star,z}[r]=\theta_\star[i_r].
\]
Thus the two \(r\)-th coordinates inherit the same semantic name from the original
FABLE dictionary.  Substituting the definitions and using
\(S_zS_z^\top=Z\) gives
\[
\begin{aligned}
    \phi_z(C,a)^\top\theta_{\star,z}
    &=
    \phi_{\mathrm{fac}}(C,a)^\top
    S_zS_z^\top\theta_\star\\
    &=
    \phi_{\mathrm{fac}}(C,a)^\top Z\theta_\star.
\end{aligned}
\]
\end{proof}

\begin{lemma}[Active prior induced by full-factorized onboarding]
\label{lem:pure-spca-active-prior}
If
\[
    \theta\mid Q\sim\Normal(\mu_Q,\Sigma_Q)
\]
is the onboarding posterior in the original \(d_f\)-dimensional FABLE coordinate
space, then the subfactorized parameter \(\theta_z=S_z^\top\theta\) satisfies
\[
    \theta_z\mid Q\sim\Normal(\mu_{0,z},\Sigma_{0,z}),
\]
where
\begin{equation}
\label{eq:pure-spca-active-prior-mean}
    \mu_{0,z}=S_z^\top\mu_Q,
\end{equation}
and
\begin{equation}
\label{eq:pure-spca-active-prior-cov}
    \Sigma_{0,z}=S_z^\top\Sigma_QS_z.
\end{equation}
The online Bayesian update for FABLE-SPCA is then the usual update in
\(\R^{d_z}\), using \(\phi_z\), \(\mu_{0,z}\), and \(\Sigma_{0,z}\).
\end{lemma}

\begin{proof}
The vector \(\theta_z=S_z^\top\theta\) is a linear transformation of a Gaussian
random vector.  Therefore it is Gaussian with mean
\[
    \E[\theta_z\mid Q]=S_z^\top\mu_Q
\]
and covariance
\[
    \operatorname{Cov}(\theta_z\mid Q)=S_z^\top\Sigma_QS_z.
\]
\end{proof}

\begin{lemma}[Compatibility with the original active-space onboarding update]
\label{lem:pure-spca-onboarding-compatibility}
Suppose the base prior is isotropic and the mask is onboarding-closed:
\[
    \operatorname{supp}(v_j)\subseteq\mathcal I(z),
    \qquad
    \forall j \text{ with }\kappa_j>0.
\]
Let
\[
    v_{j,z}=S_z^\top v_j.
\]
Then the active prior in Lemma~\ref{lem:pure-spca-active-prior} is the same
posterior that would be obtained by applying the original onboarding equations
directly in the \(d_z\)-dimensional subfactorized space:
\[
    \Sigma_{0,z}^{-1}
    =
    \lambda_0I_{d_z}
    +
    \sum_{j=1}^J\kappa_jv_{j,z}v_{j,z}^\top,
\]
and
\[
\begin{aligned}
    \mu_{0,z}
    &=
    \Sigma_{0,z}
    \left(
        \lambda_0\mu_{\mathrm{base},z}
        +
        \sum_{j=1}^J\kappa_jv_{j,z}u_j
    \right).
\end{aligned}
\]
Here \(\mu_{\mathrm{base},z}=S_z^\top\mu_{\mathrm{base}}\).
\end{lemma}

\begin{proof}
Under the onboarding-closed condition, each \(v_j\) has zero coordinates outside
\(\mathcal I(z)\).  Hence the full-factorized precision matrix
\[
    \Sigma_Q^{-1}=\lambda_0I_{d_f}+\sum_{j=1}^J\kappa_jv_jv_j^\top
\]
is block diagonal with respect to the selected coordinates and their complement.
Its selected-coordinate block is
\[
    \lambda_0I_{d_z}
    +
    \sum_{j=1}^J\kappa_jv_{j,z}v_{j,z}^\top.
\]
Define the full-space information vector
\[
    h_Q
    =
    \lambda_0\mu_{\mathrm{base}}
    +
    \sum_{j=1}^J\kappa_jv_ju_j.
\]
It restricts in the same way:
\[
\begin{aligned}
    S_z^\top h_Q
    =
    \lambda_0\mu_{\mathrm{base},z}
    +
    \sum_{j=1}^J\kappa_jv_{j,z}u_j.
\end{aligned}
\]
Thus updating in the original \(d_f\)-dimensional factorized space and then
restricting to \(\mathcal I(z)\) gives the same selected-coordinate posterior as
updating directly in the subfactorized active space.
\end{proof}

\paragraph{Coordinate-sparse PCA objective.}
A subfactorization omits some existing FABLE coordinates.  The omitted-coordinate
contribution is measured in the current factorized coordinate system:
Define
\[
    \eta_z(C,a,\theta)
    =
    \phi_{\mathrm{fac}}(C,a)^\top(I_{d_f}-Z)\theta.
\]
\begin{equation}
\label{eq:pure-spca-omitted-error}
    \mathcal E_Q(z)
    =
    \E_{\theta\sim\Normal(\mu_Q,\Sigma_Q)}
    \E_{C\sim\nu_Q,\ a\sim\pi_0(\cdot\mid C)}
    \left[\eta_z(C,a,\theta)^2\right].
\end{equation}
This is the PCA-style reconstruction error, but for the scalar personalized
reward contribution rather than for the feature vector alone.

\begin{lemma}[Hadamard form of the subfactorization error]
\label{lem:pure-spca-hadamard-error}
Let \(G_Q\) be defined by Equation~\eqref{eq:pure-spca-GQ}, let
\(M_Q=\Sigma_Q+\mu_Q\mu_Q^\top\), and define the coordinate-sparse PCA relevance
matrix
\begin{equation}
\label{eq:pure-spca-relevance-matrix}
    H_Q=G_Q\circ M_Q,
\end{equation}
where \(\circ\) denotes the Hadamard product.  Then
\begin{equation}
\label{eq:pure-spca-error-binary-quadratic}
    \mathcal E_Q(z)
    =
    (\mathbf 1-z)^\top H_Q(\mathbf 1-z).
\end{equation}
Moreover, since \(G_Q\succeq 0\) and \(M_Q\succeq 0\), we have
\(H_Q\succeq 0\) by the Schur product theorem.
\end{lemma}

\begin{proof}
Let \(A_z=I_{d_f}-Z\) and \(\varphi=\phi_{\mathrm{fac}}(C,a)\).  The squared
omitted contribution is
\[
    (\varphi^\top A_z\theta)^2
    =
    \theta^\top A_z\varphi\varphi^\top A_z\theta.
\]
Taking expectation over \((C,a)\) gives
\[
    \E_{C,a}[(\varphi^\top A_z\theta)^2]
    =
    \theta^\top A_zG_QA_z\theta.
\]
Taking expectation over \(\theta\mid Q\) and using
\(\E[\theta\theta^\top\mid Q]=M_Q\) yields
\[
    \mathcal E_Q(z)
    =
    \operatorname{tr}(A_zG_QA_zM_Q).
\]
Because \(A_z\) is diagonal with diagonal entries \(1-z_i\), this trace equals
\[
\begin{aligned}
    \mathcal E_Q(z)
    &=
    \sum_{i=1}^{d_f}\sum_{j=1}^{d_f}
    (1-z_i)(1-z_j)H_Q[i,j]\\
    &=
    (\mathbf 1-z)^\top H_Q(\mathbf 1-z).
\end{aligned}
\]
The positive semidefiniteness of \(H_Q\) follows from the Schur product theorem.
\end{proof}

\begin{proposition}[Pure coordinate-sparse PCA selection]
\label{prop:pure-coordinate-sparse-pca-selection}
For a fixed active dimension \(k\), define the admissible mask family
\[
    \mathfrak Z_k(Q)
    \subseteq
    \left\{
        z\in\{0,1\}^{d_f}:\mathbf 1^\top z=k
    \right\},
\]
where \(\mathfrak Z_k(Q)\) may include must-keep constraints for trusted
onboarding coordinates, forbidden coordinates from hard user constraints, and
parent-child constraints requiring main-effect blocks before higher-order
interaction blocks.  The pure coordinate-sparse PCA subfactorization is
\begin{equation}
\label{eq:pure-coordinate-sparse-pca-problem}
    z_k^\star
    \in
    \operatorname*{arg\,min}_{z\in\mathfrak Z_k(Q)}
    \mathcal E_Q(z).
\end{equation}
Equivalently, \(z_k^\star\) maximizes the retained onboarding-weighted
variance
\begin{equation}
\label{eq:pure-spca-retained-gain}
    \operatorname{Gain}_Q(z)
    =
    \mathcal E_Q(0)-\mathcal E_Q(z)
    =
    2z^\top H_Q\mathbf 1-z^\top H_Qz.
\end{equation}
Thus selecting \(S_z\) is a cardinality-constrained coordinate-sparse PCA problem
inside the already-factorized FABLE coordinate system.  If \(H_Q\) is diagonal,
then \(z_k^\star\) keeps the \(k\) admissible coordinates with the largest
scores \(H_Q[i,i]=G_Q[i,i]M_Q[i,i]\).
This objective is a coordinate-restricted form of sparse principal component
selection~\citep{zou2006sparse,daspremont2007direct}.
\end{proposition}

\begin{proof}
For fixed \(k\), minimizing \(\mathcal E_Q(z)\) is equivalent to maximizing
\(\mathcal E_Q(0)-\mathcal E_Q(z)\), because \(\mathcal E_Q(0)\) does not depend
on \(z\).  Using Equation~\eqref{eq:pure-spca-error-binary-quadratic},
\[
\begin{aligned}
    \mathcal E_Q(0)-\mathcal E_Q(z)
    &=
    \mathbf 1^\top H_Q\mathbf 1
    -
    (\mathbf 1-z)^\top H_Q(\mathbf 1-z) \\
    &=
    2z^\top H_Q\mathbf 1-z^\top H_Qz.
\end{aligned}
\]
If \(H_Q\) is diagonal, then
\[
    \mathcal E_Q(z)=\sum_{i=1}^{d_f}(1-z_i)H_Q[i,i],
\]
so the best budget-\(k\) mask keeps the \(k\) largest diagonal entries subject to
admissibility constraints.
\end{proof}

\begin{remark}[Coordinate-preserving subfactorization]
Any selection matrix \(S\in\{0,1\}^{d_f\times d}\) produces the representation
\(S^\top\phi_{\mathrm{fac}}(C,a)\). FABLE-SPCA remains in this representation
class and specifies its selection criterion by minimizing the
onboarding-weighted omitted personalized prediction variance
\(\mathcal E_Q(S)\).
\end{remark}

\begin{remark}[Relation to standard sparse PCA]
Standard sparse PCA usually allows a loading matrix \(U\in\R^{d_f\times r}\) with
row sparsity,
\[
    U^\top U=I_r,
    \qquad
    \|U\|_{2,0}\le k,
\]
and maximizes a variance objective such as \(\operatorname{tr}(U^\top A_QU)\).
This can reduce approximation error because it allows rotations inside the
selected support. FABLE-SPCA uses the more restrictive choice
\(U=S_z\), so each active dimension remains one original semantic coordinate.
Therefore FABLE-SPCA is a coordinate-restricted sparse PCA method rather than a
latent sparse PCA method.
The distinction follows the regression and semidefinite formulations of sparse
PCA~\citep{zou2006sparse,daspremont2007direct}.
\end{remark}

\paragraph{Regret of a fixed sparse-PCA subfactorization.}
For the original full-factorized FABLE model, the expected reward is
\[
\begin{aligned}
    f_t(a)
    &=
    \bar b_t(a)
    +\phi_{\mathrm{fac}}(C_t,a)^\top\theta_\star.
\end{aligned}
\]
Here \(\bar b_t(a)=b(C_t,a)-\lambda\cost(C_t,a)\).
For a mask \(z\), define the projected reward
\[
\begin{aligned}
    f_{t,z}(a)
    &=
    \bar b_t(a)
    +\phi_{\mathrm{fac}}(C_t,a)^\top Z\theta_\star\\
    &=
    \bar b_t(a)+\phi_z(C_t,a)^\top\theta_{\star,z}.
\end{aligned}
\]
Let \(a_t^\star\) be the best feasible action under \(f_t\), and let
\(a_{t,z}^\star\) be the best feasible action under \(f_{t,z}\). For this
fixed-mask analysis, \(a_t\) denotes the action selected by the learner using
mask \(z\), and
\[
    R_n(z)
    =
    \sum_{t=1}^n
    \left[
        f_t(a_t^\star)-f_t(a_t)
    \right].
\]

\begin{assumption}[Onboarding-calibrated omitted coordinates]
\label{ass:pure-spca-calibration}
There exists \(\rho_Q\ge 1\) such that, for every candidate mask \(z\) and every
predictable feasible-action sequence considered by the algorithm or the
feasible oracle,
\[
    \E\left[
        \left(
            \phi_{\mathrm{fac}}(C_t,a)^\top(I_{d_f}-Z)\theta_\star
        \right)^2
        \middle| Q
    \right]
    \le
    \rho_Q\mathcal E_Q(z).
\]
\end{assumption}

\begin{theorem}[Bias--estimation regret decomposition for a fixed mask]
\label{thm:pure-spca-fixed-regret}
Suppose the residual-feedback model has conditionally centered sub-Gaussian
noise, the feasible sets are predictable and nonempty, and
Assumption~\ref{ass:pure-spca-calibration} holds.  Fix a
coordinate-preserving mask \(z\).  Suppose the FABLE-SPCA learner run in
\(\R^{d_z}\) has projected
feasible-oracle regret bounded by
\[
    B_n(z)
    \ge
    \E\left[
        \sum_{t=1}^n
        \left(
            f_{t,z}(a_{t,z}^\star)-f_{t,z}(a_t)
        \right)
        \middle| Q
    \right].
\]
Then its regret against the original full-factorized feasible oracle satisfies
\begin{equation}
\label{eq:pure-spca-fixed-regret}
\begin{aligned}
    \E[R_n(z)\mid Q]
    &\le
    n b_Q(z)+B_n(z),\\
    b_Q(z)
    &=2\sqrt{\rho_Q\mathcal E_Q(z)}.
\end{aligned}
\end{equation}
For a UCB-style subfactorized learner, one may take
\begin{equation}
\label{eq:pure-spca-UCB-B}
    B_n^{\mathrm{UCB}}(z)
    =
    c_0\beta_{n,z}^{\mathrm{UCB}}
    \sqrt{
        n d_z
        \log\left(
            1+
            \frac{nL_z^2}{\lambda_0d_z\sigma^2}
        \right)
    },
\end{equation}
where
\[
    L_z
    =
    \sup_{t\ge1,\ a\in\calA_t}
    \|\phi_z(C_t,a)\|_2,
\]
\(\beta_{n,z}^{\mathrm{UCB}}\) is the standard self-normalized UCB confidence
radius at horizon \(n\), and \(c_0>0\) is a universal constant. For
Thompson-style exploration, one may take
\[
    B_n^{\mathrm{TS}}(z)=\widetilde O(d_z^{3/2}\sqrt n).
\]
The UCB term follows the standard self-normalized analysis
\citep{abbasi2011improved}. The Thompson-sampling order follows
\citet{abeille2017linear}.
\end{theorem}

\begin{proof}
Let
\[
    \eta_{t,z}(a)
    =
    \phi_{\mathrm{fac}}(C_t,a)^\top(I_{d_f}-Z)\theta_\star
\]
be the omitted-coordinate contribution.  Then
\[
    f_t(a)=f_{t,z}(a)+\eta_{t,z}(a).
\]
For each round,
\[
\begin{aligned}
    f_t(a_t^\star)-f_t(a_t)
    &=
    \bigl[
        f_{t,z}(a_t^\star)-f_{t,z}(a_t)
    \bigr]\\
    &\quad+
    \bigl[
        \eta_{t,z}(a_t^\star)-\eta_{t,z}(a_t)
    \bigr].
\end{aligned}
\]
Because \(a_{t,z}^\star\) maximizes \(f_{t,z}\) over the same feasible set,
\[
\begin{aligned}
    f_t(a_t^\star)-f_t(a_t)
    &\le
    f_{t,z}(a_{t,z}^\star)
    -f_{t,z}(a_t)\\
    &\quad+
    |\eta_{t,z}(a_t^\star)|
    +
    |\eta_{t,z}(a_t)|,
\end{aligned}
\]
Taking
conditional expectations and summing over \(t\), the first term is bounded by
\(B_n(z)\).  For either omitted term, Jensen's inequality and
Assumption~\ref{ass:pure-spca-calibration} give
\[
    \E[|\eta_{t,z}(a)|\mid Q]
    \le
    \sqrt{\E[\eta_{t,z}(a)^2\mid Q]}
    \le
    \sqrt{\rho_Q\mathcal E_Q(z)}.
\]
Thus the omitted-coordinate terms contribute at most
\(2\sqrt{\rho_Q\mathcal E_Q(z)}\) per round, proving
Equation~\eqref{eq:pure-spca-fixed-regret}.  The displayed choices of \(B_n(z)\)
are standard linear-bandit regret bounds applied in the \(d_z\)-dimensional
subfactorized feature space.
\end{proof}

\paragraph{Regret bound for the pure sparse-PCA mask.}
The selection rule in Equation~\eqref{eq:pure-coordinate-sparse-pca-problem} does
not include the bandit estimation term.  Therefore the oracle statement is an
approximation oracle statement, not a regret-aware model-selection statement.
Let
\begin{equation}
\label{eq:pure-spca-Ek-star}
    E_k^\star(Q)=\min_{z\in\mathfrak Z_k(Q)}\mathcal E_Q(z),
    \qquad
    z_k^\star\in\operatorname*{arg\,min}_{z\in\mathfrak Z_k(Q)}\mathcal E_Q(z).
\end{equation}
If the estimation regret can be uniformly bounded over the budget class by
\begin{equation}
\label{eq:pure-spca-uniform-Bk}
    B_n(z)\le \overline B_n(k),
    \qquad
    \forall z\in\mathfrak Z_k(Q),
\end{equation}
then Theorem~\ref{thm:pure-spca-fixed-regret} gives
\begin{equation}
\label{eq:pure-spca-selected-mask-regret}
    \E[R_n(z_k^\star)\mid Q]
    \le
    2n\sqrt{\rho_QE_k^\star(Q)}+\overline B_n(k).
\end{equation}
For fixed active dimension \(k\), this mask has the smallest approximation-bias
term among admissible coordinate-preserving masks.
If \(\overline B_n(k)\) depends only on \(k\), then it is also optimal for the
upper bound within that fixed budget class.

\paragraph{Sample-efficiency and regret trade-off across budgets.}
Although \(B_n(z)\) is not part of the sparse-PCA selection objective, it is still
part of the regret analysis.  The trade-off is now expressed as a curve over the
external budget \(k\).  Let
\[
    b_k^\star(Q)=2\sqrt{\rho_QE_k^\star(Q)}.
\]
Fix \(\mathsf{alg}\in\{\mathrm{UCB},\mathrm{TS}\}\). Suppose that
\[
    \overline B_n(k)
    \le
    C_{\mathsf{alg}}g_{\mathsf{alg}}(k)\sqrt n\,L_n^{\mathrm{log}},
\]
where
\[
    g_{\mathrm{UCB}}(k)=k,
    \qquad
    g_{\mathrm{TS}}(k)=k^{3/2},
\]
\(C_{\mathsf{alg}}>0\) is an algorithm-dependent constant, and
\(L_n^{\mathrm{log}}\) contains logarithmic factors and fixed problem
constants. Then
Equation~\eqref{eq:pure-spca-selected-mask-regret} implies
\begin{equation}
\label{eq:pure-spca-average-regret-bound}
    \frac{1}{n}\E[R_n(z_k^\star)\mid Q]
    \le
    b_k^\star(Q)
    +
    \frac{
        C_{\mathsf{alg}}g_{\mathsf{alg}}(k)L_n^{\mathrm{log}}
    }{\sqrt n}.
\end{equation}

\begin{theorem}[Sample-efficiency threshold for pure sparse-PCA masks]
\label{thm:pure-spca-sample-efficiency-threshold}
Fix a target average regret level \(\epsilon>0\).  If
\[
    b_k^\star(Q)<\epsilon,
\]
then FABLE-SPCA using the population sparse-PCA mask \(z_k^\star\) reaches average
regret at most \(\epsilon\) after
\begin{equation}
\label{eq:pure-spca-sample-complexity-mask}
    n
    =
    \widetilde O
    \left(
        \frac{g_{\mathsf{alg}}(k)^2}
        {(\epsilon-b_k^\star(Q))^2}
    \right)
\end{equation}
rounds.  The full FABLE mask \(z_{\mathrm{full}}=\mathbf 1_{d_f}\) has
\(b_{d_f}^\star(Q)=0\) and therefore requires
\begin{equation}
\label{eq:pure-spca-sample-complexity-full}
    n_{\mathrm{full}}(\epsilon)
    =
    \widetilde O
    \left(
        \frac{g_{\mathsf{alg}}(d_f)^2}{\epsilon^2}
    \right)
\end{equation}
rounds.  Hence the budget-\(k\) sparse-PCA subfactorization is more sample
efficient than full FABLE for accuracy level \(\epsilon\) whenever, up to
logarithmic factors,
\begin{equation}
\label{eq:pure-spca-sample-efficiency-condition}
    b_k^\star(Q)
    <
    \epsilon
    \left(
        1-\frac{g_{\mathsf{alg}}(k)}{g_{\mathsf{alg}}(d_f)}
    \right).
\end{equation}
\end{theorem}

\begin{proof}
By Equation~\eqref{eq:pure-spca-average-regret-bound}, it is sufficient that
\[
    b_k^\star(Q)
    +
    \frac{
        C_{\mathsf{alg}}g_{\mathsf{alg}}(k)L_n^{\mathrm{log}}
    }{\sqrt n}
    \le \epsilon.
\]
Since \(b_k^\star(Q)<\epsilon\), this holds whenever
\[
    \sqrt n
    \ge
    \frac{
        C_{\mathsf{alg}}g_{\mathsf{alg}}(k)L_n^{\mathrm{log}}
    }{\epsilon-b_k^\star(Q)}.
\]
Suppressing logarithmic factors gives
Equation~\eqref{eq:pure-spca-sample-complexity-mask}.  For the full mask, the
omitted-coordinate error is zero, so \(b_{d_f}^\star(Q)=0\), and the same calculation
gives Equation~\eqref{eq:pure-spca-sample-complexity-full}.  Comparing the two
sufficient sample sizes yields
\[
    \frac{g_{\mathsf{alg}}(k)}{\epsilon-b_k^\star(Q)}
    <
    \frac{g_{\mathsf{alg}}(d_f)}{\epsilon},
\]
which rearranges to Equation~\eqref{eq:pure-spca-sample-efficiency-condition}.
\end{proof}

\begin{corollary}[Finite-horizon regret crossover]
\label{cor:pure-spca-regret-crossover}
Ignore logarithmic factors and write
\[
    \overline B_n(k)
    \approx
    C_{\mathsf{alg}}g_{\mathsf{alg}}(k)\sqrt n.
\]
Let \(z_{\mathrm{full}}=\mathbf 1_{d_f}\). If \(b_k^\star(Q)>0\), then the bound for
the budget-\(k\) pure sparse-PCA mask is smaller than the full FABLE bound
whenever
\begin{equation}
\label{eq:pure-spca-crossover-horizon}
    n
    <
    n_\times(k)
    :=
    \left(
        \frac{
            C_{\mathsf{alg}}
            (g_{\mathsf{alg}}(d_f)-g_{\mathsf{alg}}(k))
        }{b_k^\star(Q)}
    \right)^2.
\end{equation}
If \(b_k^\star(Q)=0\) and \(k<d_f\), then the subfactorized bound is no larger than
the full bound for all horizons and is strictly smaller whenever the estimation
term is strictly increasing in dimension.
\end{corollary}

\begin{proof}
The budget-\(k\) bound is smaller than the full bound if
\[
    nb_k^\star(Q)+C_{\mathsf{alg}}g_{\mathsf{alg}}(k)\sqrt n
    <
    C_{\mathsf{alg}}g_{\mathsf{alg}}(d_f)\sqrt n.
\]
For \(b_k^\star(Q)>0\), dividing by \(\sqrt n\) and rearranging gives
Equation~\eqref{eq:pure-spca-crossover-horizon}. If \(b_k^\star(Q)=0\), the
inequality reduces to
\(g_{\mathsf{alg}}(k)<g_{\mathsf{alg}}(d_f)\), which holds whenever
\(k<d_f\) and the estimation term is strictly increasing in dimension.
\end{proof}

\begin{remark}[Interpretation of the trade-off]
The sparse-PCA objective itself only tries to preserve onboarding-weighted
personalized prediction variance.  The sample-efficiency advantage appears after
choosing an external budget \(k\): smaller \(k\) gives lower estimation cost but
larger omitted-coordinate bias.  Thus pure sparse-PCA subfactorization and the
earlier coordinate-preserving subfactorization have the same fixed-budget theory;
the regret and sample-efficiency theorems explain when a smaller fixed budget is
preferable to the full FABLE feature map.
\end{remark}

\paragraph{Block coordinate-sparse PCA integer program.}
In implementation, masks are often selected at the block level.  Let
\[
    \mathcal B=\{B_1,\ldots,B_H\}
\]
be semantic blocks of the existing FABLE coordinates, such as memory main effects,
tool main effects, style main effects, task-style interactions, task-tool
interactions, need interactions, and action-component interactions.  Let
\(q_h\in\{0,1\}\) indicate whether block \(B_h\) is retained, and let
\[
    z(q)_i=1
    \quad\Longleftrightarrow\quad
    i\in\bigcup_{h:q_h=1}B_h.
\]
The pure block sparse-PCA problem at budget \(k\) is
\begin{equation}
\label{eq:pure-spca-block-integer-program}
\begin{aligned}
    \widehat q_k
    \in
    \operatorname*{arg\,min}_{q\in\{0,1\}^H}
    \quad &
    \widehat{\mathcal E}_Q(z(q))
    =
    (\mathbf 1-z(q))^\top\widehat H_Q(\mathbf 1-z(q))
    \\
    \text{s.t.}\quad &
    d_{z(q)}=k,\\
    &q_h=1,\qquad h\in\widehat{\mathcal H}_{\mathrm{must}},\\
    &q_h=0,\qquad h\in\widehat{\mathcal H}_{\mathrm{forbid}},\\
    &q_{h'}\le q_h,\qquad (h,h')\in\mathcal D.
\end{aligned}
\end{equation}
Here
\[
    \widehat M_Q
    =
    \widehat\Sigma_Q
    +
    \widehat\mu_Q\widehat\mu_Q^\top,
    \qquad
    \widehat H_Q
    =
    \widehat G_Q\circ\widehat M_Q,
\]
where \((\widehat\mu_Q,\widehat\Sigma_Q)\) are the onboarding posterior
moments obtained from the LLM-estimated directions, responses, and confidence
weights via Equations~\eqref{eq:pure-spca-prior-precision}--%
\eqref{eq:pure-spca-prior-mean}. The set
\(\widehat{\mathcal H}_{\mathrm{must}}\) contains blocks touched by trusted
onboarding directions, and \(\mathcal D\) contains parent-child dependencies such
as requiring a main-effect block before its interaction block.

\paragraph{LLM-estimated version.}
The LLM enters FABLE-SPCA only through structured estimates: sparse onboarding
directions \(v_j\), responses \(u_j\), confidence weights \(\kappa_j\), and context
prototypes \(\widehat C^{(\ell)}\).  The feature dictionary, feasible set, block
library, sparse-PCA objective, and integer optimizer are fixed by the algorithm.
If the estimated sparse-PCA error is uniformly calibrated, the selected mask is
near-oracle for the pure approximation objective.

\begin{theorem}[Near-oracle sparse-PCA selection under LLM score calibration]
\label{thm:pure-spca-llm-near-oracle}
Let \(\mathcal E_Q(z)\) be the population sparse-PCA error, and let
\(\widehat{\mathcal E}_Q(z)\) be the same error computed from LLM-estimated
quantities such as \(\widehat G_Q\), \(\widehat M_Q\), extracted directions, and estimated
confidence weights.  Suppose that for all \(z\in\mathfrak Z_k(Q)\),
\[
    \left|
        \widehat{\mathcal E}_Q(z)-\mathcal E_Q(z)
    \right|
    \le
    \xi_E.
\]
If the integer optimizer returns \(\widehat z_k\) satisfying
\[
    \widehat{\mathcal E}_Q(\widehat z_k)
    \le
    \min_{z\in\mathfrak Z_k(Q)}
    \widehat{\mathcal E}_Q(z)
    +
    \epsilon_{\mathrm{opt}},
\]
then
\[
    \mathcal E_Q(\widehat z_k)
    \le
    \min_{z\in\mathfrak Z_k(Q)}
    \mathcal E_Q(z)
    +
    2\xi_E+
    \epsilon_{\mathrm{opt}}.
\]
Consequently,
\[
    \E[R_n(\widehat z_k)\mid Q]
    \le
    2n\sqrt{\rho_Q\left(E_k^\star(Q)+2\xi_E+\epsilon_{\mathrm{opt}}\right)}
    +
    \overline B_n(k),
\]
whenever the uniform estimation bound \(B_n(z)\le\overline B_n(k)\) holds over
\(\mathfrak Z_k(Q)\).
\end{theorem}

\begin{proof}
Let \(z_k^\star\in\operatorname*{arg\,min}_{z\in\mathfrak Z_k(Q)}\mathcal E_Q(z)\). By calibration,
\[
    \mathcal E_Q(\widehat z_k)
    \le
    \widehat{\mathcal E}_Q(\widehat z_k)+\xi_E.
\]
By approximate optimality,
\[
    \widehat{\mathcal E}_Q(\widehat z_k)
    \le
    \widehat{\mathcal E}_Q(z_k^\star)+
    \epsilon_{\mathrm{opt}}.
\]
By calibration again,
\[
    \widehat{\mathcal E}_Q(z_k^\star)
    \le
    \mathcal E_Q(z_k^\star)+
    \xi_E.
\]
Combining the three inequalities yields
\[
    \mathcal E_Q(\widehat z_k)
    \le
    \mathcal E_Q(z_k^\star)+2\xi_E+
    \epsilon_{\mathrm{opt}}.
\]
The regret statement follows by substituting this bound into
Theorem~\ref{thm:pure-spca-fixed-regret} and using the uniform bound
\(B_n(z)\le\overline B_n(k)\).
\end{proof}

\section{Agent-Execution Feature Dictionary and Action Catalog}
\label{app:feature-dictionary}

This appendix specifies the memory--tool--response catalog used for the
agent-execution instantiation of FABLE. The method itself requires a fixed
finite product action space, not these particular component semantics. Thus the
Math500 experiment uses rubric-criterion components, whereas the PAHF and
tau2-bench experiments use memory, information-acquisition, and response
components. The catalog below instantiates the latter representation and is not
part of the abstract problem formulation.

The representative action sets are
\[
\begin{aligned}
\calM &=
\{\text{no memory},\text{recent memory},\\
&\qquad \text{semantic memory},\text{preference memory},\\
&\qquad \text{profile summary}\},\\
\calT &=
\{\text{no tool},\text{web search},\text{file search},\\
&\qquad \text{code execution},\\
&\qquad \text{preference checker},\text{ask user}\},\\
\calS &=
\{\text{direct},\text{concise},\text{detailed},\\
&\qquad \text{step-by-step},\text{ask clarification},\\
&\qquad \text{confirm first}\}.
\end{aligned}
\]
Thus \(\calA=\calM\times\calT\times\calS\). Let
\(C_t=(k_t,r_t,g_t,p_t,q_t)\), where \(k_t\in\calK\) is task type and
\(r_t,g_t,p_t,q_t\in[0,1]\) are risk, ambiguity, memory need, and tool need.
Let \(m_0=\text{no memory}\) and \(\tau_0=\text{no tool}\). We use
\(e_{\calM}(m)\), \(e_{\calT}(\tau)\), \(e_{\calS}(s)\), and
\(e_{\calK}(k_t)\) for one-hot vectors over memory modes, tool modes, answer
styles, and task types. The reduced vectors
\(e_{\calM\setminus\{m_0\}}(m)\) and
\(e_{\calT\setminus\{\tau_0\}}(\tau)\) are zero for the corresponding null
action.

The concrete factorized feature map is
\begin{align}
\phi(C_t,a)
=
\big[
&e_{\calM}(m),
e_{\calT}(\tau),
e_{\calS}(s),
\nonumber\\
&
e_{\calK}(k_t)\otimes e_{\calS}(s),
\nonumber\\
&e_{\calK}(k_t)\otimes e_{\calT}(\tau),
p_t e_{\calM\setminus\{m_0\}}(m),
\nonumber\\
&
q_t e_{\calT\setminus\{\tau_0\}}(\tau),
\nonumber\\
&r_t e_{\calS}(s),
g_t e_{\calS}(s),
\nonumber\\
&e_{\calM}(m)\otimes e_{\calT}(\tau),
e_{\calM}(m)\otimes e_{\calS}(s),
\nonumber\\
&
e_{\calT}(\tau)\otimes e_{\calS}(s)
\big].
\label{eq:factorized-feature}
\end{align}
The first three blocks capture main effects. The next blocks capture
task--style, task--tool, and need interactions; the final blocks capture
pairwise interactions among action components. With
\(M=|\calM|\), \(T=|\calT|\), \(S=|\calS|\), and \(K=|\calK|\), the
dimension is
\[
\begin{aligned}
d={}&M+T+S+KS+KT+(M-1)+(T-1)\\
&+2S+MT+MS+TS.
\end{aligned}
\]
For the action spaces above, \(M=5\), \(T=6\), \(S=6\), and \(K=10\), so
\(d=254\).

\paragraph{Adapter implementations.}
An LLM-assisted runtime may implement the context, onboarding, and feedback
adapters by mapping raw requests to \(C_t\), free-form onboarding answers to
\((v_j,u_j,\kappa_j)\), and textual reactions to \(y_t\). The action
components, index map, and feature map remain fixed, and Algorithm 1 of the
main paper determines the action and posterior update. Any
fixed elicitation procedure that supplies \((\mu_0,\Sigma_0)\) in the same
feature space may replace the pseudo-observation construction in Equations
(2)--(3) of the main paper.

\section{Reference Default and Cost Specification}
\label{app:baseline-cost}

The following quantities give one prespecified default--cost specification for
the memory--tool--response action catalog. They are not estimated from online
feedback. Each benchmark fixes its own feedback adapter and cost scaling, as
reported in Section 5 of the main paper; the learning rule does not depend on
the particular numerical values below. Let \(c=(k,r,g,p,q)\) and
\(a=(m,\tau,s)\), with \(m_0=\text{no memory}\) and
\(\tau_0=\text{no tool}\). The default score is
\[
\begin{aligned}
b(c,a)={}&
b_{\mathrm{mem}}(p,m)+b_{\mathrm{tool}}(q,\tau)+
b_{\mathrm{amb}}(g,\tau,s)\\
&+b_{\mathrm{risk}}(r,s)+b_{\mathrm{task}}(k,s),
\end{aligned}
\]
where
\[
\resizebox{\columnwidth}{!}{$
\begin{aligned}
b_{\mathrm{mem}}(p,m)
={}&
0.18\,\mathbf 1\{p>0.5,\ m\ne m_0\}\\
&{}-
0.04\,\mathbf 1\{p\le 0.2,\ m\neq m_0\},\\
b_{\mathrm{tool}}(q,\tau)
={}&
0.18\,\mathbf 1\{q>0.5,\ \tau\ne \tau_0\}\\
&{}-
0.12\,\mathbf 1\{q<0.2,\ \tau\neq\tau_0\},\\
b_{\mathrm{amb}}(g,\tau,s)
={}&
0.16\,\mathbf 1\{g>0.6,\ \tau=\text{ask user}\\
&\qquad\qquad\text{or }s=\text{ask clarification}\}\\
&{}-
0.08\,\mathbf 1\{g>0.6,\ s=\text{direct}\},\\
b_{\mathrm{risk}}(r,s)
={}&
0.20\,\mathbf 1\{r>0.65,\ s=\text{confirm first}\}
-
0.12\,\mathbf 1\{r>0.65,\\
&\qquad\qquad s\in\{\text{direct},\text{concise}\}\},\\
b_{\mathrm{task}}(k,s)
={}&
0.08\,\mathbf 1\{k\in\{\text{coding},\text{analysis}\},\\
&\qquad\qquad s=\text{step-by-step}\}\\
&{}+
0.06\,\mathbf 1\{k\in\{\text{factual},\\
&\qquad\qquad\text{simple preference}\},\\
&\qquad\qquad s=\text{concise}\}.
\end{aligned}
$}
\]
The additive cost is
\[
\cost(c,a)=c_M(m)+c_T(\tau)+c_S(s),
\]
with
{\small
\[
\begin{aligned}
c_M(\text{no memory})&=0,\\
c_M(\text{recent memory})&=0.02,\\
c_M(\text{semantic memory})&=0.04,\\
c_M(\text{preference memory})&=0.04,\\
c_M(\text{profile summary})&=0.05,\\[1pt]
c_T(\text{no tool})&=0,\\
c_T(\text{web search})&=0.08,\\
c_T(\text{file search})&=0.06,\\
c_T(\text{code execution})&=0.10,\\
c_T(\text{preference checker})&=0.04,\\
c_T(\text{ask user})&=0.12,\\[1pt]
c_S(\text{direct})&=0,\\
c_S(\text{concise})&=0,\\
c_S(\text{detailed})&=0.05,\\
c_S(\text{step-by-step})&=0.06,\\
c_S(\text{ask clarification})&=0.10,\\
c_S(\text{confirm first})&=0.08.
\end{aligned}
\]
}
The weight \(\lambda\ge 0\) controls the trade-off between personalized reward
and operational burden, including latency, external tool calls, and user
interruptions.

\section{Preference Promotion Rule}
\label{app:promotion-rule}

A promoted preference is defined through an identifiable score contrast.
Let
\[
    \mathcal S_{\mathrm{dict}}
    =
    \operatorname{span}
    \left\{
        \phi(c,a):(c,a)\text{ is admissible}
    \right\}.
\]
Here admissibility is with respect to the fixed feature dictionary. For each
preference \(\ell\), fix an admissible reference context \(c_\ell\)
and two reference actions \(a_\ell^+\) and \(a_\ell^-\) that differ only in
the behavior being compared, and define the nonzero direction
\[
    w_\ell
    =
    \phi(c_\ell,a_\ell^+)-\phi(c_\ell,a_\ell^-)
    \in
    \mathcal S_{\mathrm{dict}}.
\]
For example, the two actions may hold memory and tool modes fixed while
comparing concise and detailed response styles. Define the residual preference
contrast
\[
    \eta_\ell
    =
    w_\ell^\top\theta_\star.
\]
If the promotion target is the total expected-utility contrast, its known
offset
\[
    \delta_\ell^0
    =
    \bar b(c_\ell,a_\ell^+)
    -
    \bar b(c_\ell,a_\ell^-)
\]
can be added to both endpoints of the confidence interval below. We state the
results for the residual contrast \(\eta_\ell\).

Under the Gaussian working posterior,
\[
    w_\ell^\top\theta
    \sim
    \Normal
    \left(
        w_\ell^\top\mu_t,
        w_\ell^\top\Sigma_t w_\ell
    \right).
\]
Let \(\Phi\) denote the standard normal cumulative distribution function and
define
\[
\begin{aligned}
    p_{\ell,t}^{+}
    &=
    \Pr(w_\ell^\top\theta>0)
    =
    \Phi\left(
        \frac{w_\ell^\top\mu_t}
        {\sqrt{w_\ell^\top\Sigma_t w_\ell}}
    \right),\\
    p_{\ell,t}^{-}
    &=
    1-p_{\ell,t}^{+}.
\end{aligned}
\]

Let
\[
    \mathcal W
    =
    \{w_1,\ldots,w_{L_{\mathrm{pref}}}\}
\]
be a finite set of prespecified nonzero contrasts. For each direction, let
\(I_{\ell,i}\in\{0,1\}\) indicate whether round \(i\) is designated
informative, and define the prespecified implementation evidence count
\[
    n_{\ell,t}
    =
    \sum_{i<t} I_{\ell,i}.
\]
At round \(t\), let
\[
    D_{\ell,t}\in\{-1,0,+1\}
\]
denote negative promotion, no promotion, or positive promotion, respectively.
A positive promotion is wrong when \(\eta_\ell\le0\), and a negative promotion
is wrong when \(\eta_\ell\ge0\). Define
\[
\begin{aligned}
    \mathcal E_{\mathrm{wrong}}
    ={}&
    \left\{
        \exists t\ge1,\ \exists\ell:
        D_{\ell,t}=+1,\ \eta_\ell\le0
    \right\}
    \\
    &\cup
    \left\{
        \exists t\ge1,\ \exists\ell:
        D_{\ell,t}=-1,\ \eta_\ell\ge0
    \right\}.
\end{aligned}
\]

Promotion updates persistent agent state for future context construction; it
does not add an observation or alter the Gaussian posterior recursion.

\subsection{Theoretical Guarantees}
\label{subsec:promotion-theory}

Fix a global error level \(\alpha\in(0,1)\) and an integer evidence threshold
\(n_{\min}\ge1\).

\begin{proposition}[Anytime confidence ellipsoid]
\label{prop:anytime-confidence-ellipsoid}
Under Assumptions 1--2 of the main paper, define
\[
    \beta_t(\alpha)
    =
    R_b
    +
    \sqrt{
        2\log\left(
            \frac{
                \det(\Lambda_t)^{1/2}
            }{
                \alpha\,\det(\Lambda_0)^{1/2}
            }
        \right)
    }.
\]
Then, with probability at least \(1-\alpha\), simultaneously for all
\(t\ge1\),
\[
    \|\mu_t-\theta_\star\|_{\Lambda_t}
    \le
    \beta_t(\alpha).
\]
Consequently, on the same event, simultaneously for every fixed
\(w\in\R^d\) and every \(t\ge1\),
\[
    \left|
        w^\top\mu_t-w^\top\theta_\star
    \right|
    \le
    \beta_t(\alpha)
    \sqrt{w^\top\Lambda_t^{-1}w}.
\]
\end{proposition}

\begin{proof}
Apply the self-normalized linear-martingale inequality to the scaled design
vectors \(x_i(a_i)/\sigma\) and scaled noise
\(\varepsilon_i/\sigma\), and combine it with
\(\|\theta_\star-\mu_0\|_{\Lambda_0}\le R_b\); see
\citet{abbasi2011improved}. The directional inequality follows from
Cauchy--Schwarz in the \(\Lambda_t\)-norm.
\end{proof}

For each preference contrast, define
\[
    \widehat\eta_{\ell,t}
    =
    w_\ell^\top\mu_t,
\]
\[
    \operatorname{rad}_{\ell,t}
    =
    \beta_t(\alpha)
    \sqrt{w_\ell^\top\Lambda_t^{-1}w_\ell},
\]
and
\[
    \mathrm{CI}_{\ell,t}
    =
    \left[
        \widehat\eta_{\ell,t}-\operatorname{rad}_{\ell,t},
        \widehat\eta_{\ell,t}+\operatorname{rad}_{\ell,t}
    \right].
\]
Define the anytime-calibrated posterior threshold
\[
    \rho_t(\alpha)
    =
    \Phi\bigl(\beta_t(\alpha)\bigr).
\]
Because \(\Sigma_t=\Lambda_t^{-1}\), for every nonzero \(w_\ell\),
\[
    p_{\ell,t}^{+}>\rho_t(\alpha)
    \quad\Longleftrightarrow\quad
    \inf\mathrm{CI}_{\ell,t}>0,
\]
and
\[
    p_{\ell,t}^{-}>\rho_t(\alpha)
    \quad\Longleftrightarrow\quad
    \sup\mathrm{CI}_{\ell,t}<0.
\]
FABLE uses this anytime-calibrated rule:
\[
    D_{\ell,t}=+1
    \quad\text{if}\quad
    p_{\ell,t}^{+}>\rho_t(\alpha)
    \quad\text{and}\quad
    n_{\ell,t}\ge n_{\min},
\]
\[
    D_{\ell,t}=-1
    \quad\text{if}\quad
    p_{\ell,t}^{-}>\rho_t(\alpha)
    \quad\text{and}\quad
    n_{\ell,t}\ge n_{\min},
\]
and sets \(D_{\ell,t}=0\) otherwise. Equivalently, promotion occurs only
when the corresponding anytime confidence interval lies strictly on one side
of zero and the evidence-count requirement is met.

\begin{theorem}[Anytime-valid false-promotion control]
\label{thm:anytime-promotion}
Under Assumptions 1--2 of the main paper, the confidence-sequence promotion
rule satisfies
\[
    \Pr(\mathcal E_{\mathrm{wrong}})
    \le
    \alpha.
\]
\end{theorem}

\begin{proof}
On the simultaneous event in
Proposition~\ref{prop:anytime-confidence-ellipsoid}, every
\(\eta_\ell\) belongs to \(\mathrm{CI}_{\ell,t}\) for every \(t\). If the
rule makes a positive promotion, the entire interval is positive, hence
\(\eta_\ell>0\). If it makes a negative promotion, the entire interval is
negative, hence \(\eta_\ell<0\). Therefore no wrong promotion occurs on the
coverage event, whose complement has probability at most \(\alpha\).
\end{proof}

\begin{theorem}[Time to preference promotion]
\label{thm:promotion-time}
Fix a horizon \(n\) and a direction \(w_\ell\) with
\[
    \Delta_\ell
    =
    |\eta_\ell|
    >0.
\]
Let
\[
    B_n
    =
    \max_{1\le t\le n}\beta_t(\alpha).
\]
Assume there exist constants \(c_\ell>0\) and \(\gamma_\ell>0\) such that,
for all \(1\le t\le n\),
\[
    w_\ell^\top\Lambda_t^{-1}w_\ell
    \le
    \frac{c_\ell}{\gamma_\ell(1+n_{\ell,t})}.
\]
Define the required number of informative observations
\[
    N_\ell
    =
    \max\left\{
        n_{\min},
        \left\lceil
            \frac{16c_\ell B_n^2}
            {\gamma_\ell\Delta_\ell^2}
        \right\rceil
    \right\}
\]
and the first round at which this count is reached,
\[
    T_\ell
    =
    \inf\left\{
        t\le n:
        n_{\ell,t}\ge N_\ell
    \right\}.
\]
If \(T_\ell<\infty\), then on the simultaneous confidence event of
Proposition~\ref{prop:anytime-confidence-ellipsoid}, direction \(w_\ell\) is
promoted with the correct sign no later than round \(T_\ell\).
\end{theorem}

\begin{proof}
At round \(T_\ell\), the variance-decay assumption gives
\[
\begin{aligned}
    \operatorname{rad}_{\ell,T_\ell}
    &\le
    B_n
    \sqrt{
        \frac{c_\ell}
        {\gamma_\ell(1+n_{\ell,T_\ell})}
    }
    \\
    &\le
    \frac{\Delta_\ell}{4}.
\end{aligned}
\]
On the simultaneous confidence event,
\[
    |\widehat\eta_{\ell,T_\ell}-\eta_\ell|
    \le
    \operatorname{rad}_{\ell,T_\ell}.
\]
If \(\eta_\ell>0\), then
\[
    \inf\mathrm{CI}_{\ell,T_\ell}
    \ge
    \eta_\ell-2\operatorname{rad}_{\ell,T_\ell}
    \ge
    \frac{\Delta_\ell}{2}
    >0.
\]
If \(\eta_\ell<0\), the symmetric argument gives
\[
    \sup\mathrm{CI}_{\ell,T_\ell}
    \le
    -\frac{\Delta_\ell}{2}
    <0.
\]
Because \(n_{\ell,T_\ell}\ge N_\ell\ge n_{\min}\), the evidence-count
condition also holds, so the rule promotes the correct sign by round
\(T_\ell\).
\end{proof}

\clearpage
\input{appendix/fable_deepseek_v4_two_direction}
\clearpage

\bibliography{references}

\end{document}

%% file: appendix/fable_theorem1_proof.tex
\section{Proof of Theorem 1}
\label{app:fable-ts-regret-proof}

\begin{proof}
We work under the unit-noise normalization
\(\sigma^2=1\) specified for the regret analysis. We adapt the
saturated-action argument of \citet{agrawal2013thompson}, while
accounting for the general Gaussian initialization and the known
context-dependent score \(\bar b_t(a)\).

\paragraph{Centering and whitening the initialization.}
For every round \(t\) and feasible action \(a\in\calA_t\), define
\[
    z_t(a)
    :=
    \Lambda_0^{-1/2}x_t(a),
    \qquad
    \vartheta_\star
    :=
    \Lambda_0^{1/2}
    (\theta_\star-\mu_0).
\]
Assumption 2 of the main paper gives
\[
    \|\vartheta_\star\|_2
    =
    \|\theta_\star-\mu_0\|_{\Lambda_0}
    \le
    R_b.
\]
Moreover, since
\(\Lambda_0\succeq\lambda_0I_d\) and
\(\|x_t(a)\|_2\le L_x\),
\[
    \|z_t(a)\|_2
    \le
    \frac{L_x}{\sqrt{\lambda_0}}
    =:
    L_0.
\]

Define the recentered known score
\[
    \bar b_t^{\,0}(a)
    :=
    \bar b_t(a)+x_t(a)^\top\mu_0.
\]
Then the conditional mean reward can be written as
\[
    f_t(a)
    =
    \bar b_t^{\,0}(a)
    +
    z_t(a)^\top\vartheta_\star.
\]
Thus the nonzero prior mean is absorbed into a known
round-dependent score and does not introduce an additional unknown
parameter.

For the selected actions, write
\[
    x_s=x_s(a_s),
    \qquad
    z_s=z_s(a_s),
\]
and define
\[
    V_t
    :=
    I_d+\sum_{s<t}z_sz_s^\top.
\]
Since the algorithmic initialization satisfies
\(\Lambda_1=\Lambda_0\), the round-\(t\) posterior precision is
\[
    \Lambda_t
    =
    \Lambda_0
    +
    \sum_{s<t}x_sx_s^\top,
\]
we have
\[
    V_t
    =
    \Lambda_0^{-1/2}
    \Lambda_t
    \Lambda_0^{-1/2}.
\]

Let
\[
    \check y_s
    :=
    y_s-\bar b_s(a_s)-x_s^\top\mu_0.
\]
By Assumption 1 of the main paper,
\[
    \check y_s
    =
    z_s^\top\vartheta_\star+\varepsilon_s.
\]
The posterior recursion in
Proposition~\ref{prop:residual-update} therefore gives
\[
    \widehat\vartheta_t
    :=
    \Lambda_0^{1/2}(\mu_t-\mu_0)
    =
    V_t^{-1}
    \sum_{s<t}z_s\check y_s.
\]

Likewise, define the transformed Thompson sample
\[
    \widetilde\vartheta_t
    :=
    \Lambda_0^{1/2}
    (\widetilde\theta_t-\mu_0).
\]
Conditional on the information \(\mathcal F_t\) available before the
round-\(t\) Thompson sample is drawn,
\[
    \widetilde\vartheta_t
    \sim
    \Normal\!\left(
        \widehat\vartheta_t,
        \nu_t^2V_t^{-1}
    \right).
\]
Moreover,
\[
\begin{aligned}
    \bar b_t(a)+x_t(a)^\top\widetilde\theta_t
    &=
    \bar b_t^{\,0}(a)
    +
    z_t(a)^\top\widetilde\vartheta_t.
\end{aligned}
\]
Consequently, centering and whitening leave every sampled score, the
selected action \(a_t\), and the regret \(R_n\) unchanged.

For every \(a\in\calA_t\), define the estimated and sampled total scores
\[
\begin{aligned}
    \widehat f_t(a)
    &:=
    \bar b_t^{\,0}(a)
    +
    z_t(a)^\top\widehat\vartheta_t,\\
    \widetilde f_t(a)
    &:=
    \bar b_t^{\,0}(a)
    +
    z_t(a)^\top\widetilde\vartheta_t.
\end{aligned}
\]
Then
\[
    a_t
    \in
    \operatorname*{arg\,max}_{a\in\calA_t}
    \widetilde f_t(a).
\]
The known score cancels from both relevant deviations:
\[
    \widehat f_t(a)-f_t(a)
    =
    z_t(a)^\top
    (\widehat\vartheta_t-\vartheta_\star),
\]
and
\[
    \widetilde f_t(a)-\widehat f_t(a)
    =
    z_t(a)^\top
    (\widetilde\vartheta_t-\widehat\vartheta_t).
\]

Because the feedback adapter maps the outcome of every feasible action to
\([-1,1]\), its conditional mean satisfies
\[
    f_t(a)\in[-1,1],
    \qquad
    \forall t,\quad a\in\calA_t.
\]
Hence, for every \(t\) and \(a\in\calA_t\),
\[
    0
    \le
    \Delta_t(a)
    :=
    f_t(a_t^\star)-f_t(a)
    \le
    2
    =:
    C_\Delta.
\]

\paragraph{Concentration of the estimated and sampled scores.}
Define
\[
    s_t(a)
    :=
    \sqrt{
        z_t(a)^\top V_t^{-1}z_t(a)
    }.
\]
By the self-normalized linear-martingale inequality of
\citet{abbasi2011improved}, with probability at least
\(1-\delta/3\), simultaneously for every \(t\ge1\),
\[
    \|\widehat\vartheta_t-\vartheta_\star\|_{V_t}
    \le
    R_b
    +
    \sqrt{
        \log\det(V_t)
        +
        2\log\left(\frac{3}{\delta}\right)
    }.
\]
Since
\[
    \det(V_t)
    \le
    \left(
        1+
        \frac{(t-1)L_0^2}{d}
    \right)^d,
\]
define the deterministic radius
\[
    \ell_t
    :=
    R_b
    +
    \sqrt{
        d\log\left(
            1+\frac{(t-1)L_0^2}{d}
        \right)
        +
        2\log\left(\frac{3}{\delta}\right)
    }.
\]
On the same event,
\[
    \left|
        \widehat f_t(a)-f_t(a)
    \right|
    \le
    \ell_t s_t(a),
    \qquad
    \forall t\ge1,\quad
    \forall a\in\calA_t.
\]
Let \(E_t^\mu\) denote this round-\(t\) event.

Conditional on \(\mathcal F_t\), define
\[
    G_t
    :=
    \nu_t^{-1}
    V_t^{1/2}
    \left(
        \widetilde\vartheta_t-\widehat\vartheta_t
    \right).
\]
Then
\[
    G_t\mid\mathcal F_t
    \sim
    \Normal(0,I_d).
\]
Set
\[
    \chi_t
    :=
    \sqrt d+\sqrt{4\log(t+1)}.
\]
A standard Gaussian norm bound gives
\[
    \Pr\left(
        \|G_t\|_2>\chi_t
        \,\middle|\,
        \mathcal F_t
    \right)
    \le
    \frac{1}{(t+1)^2}.
\]
Consequently, the event
\[
    E_t^\theta
    :=
    \left\{
        \left|
            \widetilde f_t(a)-\widehat f_t(a)
        \right|
        \le
        \chi_t\nu_t s_t(a),
        \quad
        \forall a\in\calA_t
    \right\}
\]
satisfies
\[
    \Pr(
        E_t^\theta
        \mid
        \mathcal F_t
    )
    \ge
    1-\frac{1}{(t+1)^2}.
\]
Define
\[
    \Gamma_t
    :=
    \ell_t+\chi_t\nu_t.
\]

\paragraph{Saturated and unsaturated actions.}
Call an action \(a\in\calA_t\) saturated at round \(t\) if
\[
    \Delta_t(a)>\Gamma_ts_t(a),
\]
and let \(\mathcal C_t\) denote the set of saturated actions.
The optimal action \(a_t^\star\) is always unsaturated.

On \(E_t^\mu\cap E_t^\theta\), every saturated action satisfies
\[
\begin{aligned}
    \widetilde f_t(a)
    &\le
    f_t(a)+\Gamma_ts_t(a)\\
    &<
    f_t(a_t^\star).
\end{aligned}
\]
Furthermore, on \(E_t^\mu\),
\[
    \widehat f_t(a_t^\star)
    -
    f_t(a_t^\star)
    \ge
    -\ell_t s_t(a_t^\star).
\]

Since \(\delta\le1/2\), the definitions of \(\ell_t\) and \(\nu_t\)
imply that
\[
    \kappa
    :=
    \sup_{t\ge1}
    \frac{\ell_t}{\nu_t}
    <
    \infty,
\]
where \(\kappa\) depends only on the fixed problem constants.
Let
\[
    p_0
    :=
    \Pr(Z\ge\kappa),
    \qquad
    Z\sim\Normal(0,1).
\]
Then \(p_0>0\). Whenever
\(s_t(a_t^\star)>0\), Gaussian anti-concentration gives
\[
    \Pr\left(
        \widetilde f_t(a_t^\star)
        \ge
        f_t(a_t^\star)
        \,\middle|\,
        \mathcal F_t
    \right)
    \ge
    p_0
\]
on \(E_t^\mu\).

If \(s_t(a_t^\star)=0\), then on \(E_t^\mu\),
\[
    \widehat f_t(a_t^\star)
    =
    f_t(a_t^\star),
\]
and the Thompson sample also has zero variance in this direction, so
\[
    \widetilde f_t(a_t^\star)
    =
    f_t(a_t^\star).
\]
Therefore, in either case, on \(E_t^\mu\),
\[
    \Pr\left(
        a_t\notin\mathcal C_t
        \,\middle|\,
        \mathcal F_t
    \right)
    \ge
    p_0-\frac{1}{(t+1)^2}.
\]
Choose a fixed integer \(t_0\) such that
\[
    \frac{1}{(t+1)^2}
    \le
    \frac{p_0}{2},
    \qquad
    \forall t\ge t_0.
\]
Then, for every \(t\ge t_0\),
\[
    \Pr\left(
        a_t\notin\mathcal C_t
        \,\middle|\,
        \mathcal F_t
    \right)
    \ge
    \frac{p_0}{2}.
\]

Let
\[
    \bar a_t
    \in
    \operatorname*{arg\,min}_{a\notin\mathcal C_t}
    s_t(a).
\]
On \(E_t^\mu\cap E_t^\theta\), sampled-score optimality gives
\[
    \widetilde f_t(a_t)
    \ge
    \widetilde f_t(\bar a_t).
\]
Since \(\bar a_t\) is unsaturated,
\[
\begin{aligned}
    \Delta_t(a_t)
    &=
    \Delta_t(\bar a_t)
    +
    f_t(\bar a_t)-f_t(a_t)\\
    &\le
    \Gamma_ts_t(\bar a_t)
    +
    \widetilde f_t(\bar a_t)
    -
    \widetilde f_t(a_t)\\
    &\qquad
    +
    \Gamma_ts_t(\bar a_t)
    +
    \Gamma_ts_t(a_t)\\
    &\le
    2\Gamma_ts_t(\bar a_t)
    +
    \Gamma_ts_t(a_t).
\end{aligned}
\]

Moreover,
\[
\begin{aligned}
    \E[
        s_t(a_t)
        \mid
        \mathcal F_t
    ]
    &\ge
    s_t(\bar a_t)
    \Pr\left(
        a_t\notin\mathcal C_t
        \,\middle|\,
        \mathcal F_t
    \right)\\
    &\ge
    \frac{p_0}{2}
    s_t(\bar a_t),
\end{aligned}
\]
and hence
\[
    s_t(\bar a_t)
    \le
    \frac{2}{p_0}
    \E[
        s_t(a_t)
        \mid
        \mathcal F_t
    ].
\]

Let
\[
    \Delta_t'
    :=
    \Delta_t(a_t)
    \mathbf 1\{E_t^\mu\}.
\]
Because \(E_t^\mu\) is \(\mathcal F_t\)-measurable, the preceding
bounds, together with
\(\Delta_t(a_t)\le C_\Delta\), imply that for \(t\ge t_0\),
\[
\begin{aligned}
    \E[
        \Delta_t'
        \mid
        \mathcal F_t
    ]
    \le{}&
    \left(
        1+\frac{4}{p_0}
    \right)
    \Gamma_t
    \E[
        s_t(a_t)
        \mid
        \mathcal F_t
    ]\\
    &+
    \frac{C_\Delta}{(t+1)^2}.
\end{aligned}
\]
The finitely many rounds \(t<t_0\) contribute at most
\(t_0C_\Delta\).

\paragraph{Summing the conditional regret.}
Since
\[
    0\le \Delta_t'\le C_\Delta,
\]
the martingale-difference form of the Azuma--Hoeffding inequality gives,
with probability at least \(1-\delta/3\),
\[
    \sum_{t=1}^n\Delta_t'
    \le
    \sum_{t=1}^n
    \E[
        \Delta_t'
        \mid
        \mathcal F_t
    ]
    +
    C_\Delta
    \sqrt{
        2n\log\left(\frac{3}{\delta}\right)
    }.
\]

Similarly, since
\[
    0
    \le
    s_t(a_t)
    \le
    L_0,
\]
another martingale-difference application gives, with probability at
least \(1-\delta/3\),
\[
\begin{aligned}
    \sum_{t=1}^n
    \E[
        s_t(a_t)
        \mid
        \mathcal F_t
    ]
    \le{}&
    \sum_{t=1}^n s_t(a_t)\\
    &+
    L_0
    \sqrt{
        2n\log\left(\frac{3}{\delta}\right)
    }.
\end{aligned}
\]

The sequence \(\Gamma_t\) is nondecreasing. Therefore, on the intersection
of these two martingale events and the simultaneous event
\(\bigcap_{t\ge1}E_t^\mu\), whose probability is at least
\(1-\delta\), we have
\[
\begin{aligned}
    R_n
    \le{}&
    t_0C_\Delta\\
    &+
    \left(
        1+\frac{4}{p_0}
    \right)
    \Gamma_n
    \left[
        \sum_{t=1}^n s_t(a_t)
        +
        L_0
        \sqrt{
            2n\log\left(\frac{3}{\delta}\right)
        }
    \right]\\
    &+
    C_\Delta
    \sum_{t=1}^n\frac{1}{(t+1)^2}\\
    &+
    C_\Delta
    \sqrt{
        2n\log\left(\frac{3}{\delta}\right)
    }.
\end{aligned}
\]

\paragraph{Elliptical potential and final rate.}
By the matrix determinant lemma,
\[
    \log\det(V_{n+1})
    =
    \sum_{t=1}^n
    \log\left(
        1+s_t(a_t)^2
    \right).
\]
Since \(s_t(a_t)^2\le L_0^2\),
\[
    s_t(a_t)^2
    \le
    (1+L_0^2)
    \log\left(
        1+s_t(a_t)^2
    \right).
\]
Consequently,
\[
\begin{aligned}
    \sum_{t=1}^n s_t(a_t)
    &\le
    \sqrt{
        n\sum_{t=1}^n s_t(a_t)^2
    }\\
    &\le
    \sqrt{
        n(1+L_0^2)
        \log\det(V_{n+1})
    }\\
    &\le
    \sqrt{
        n(1+L_0^2)d
        \log\left(
            1+\frac{nL_0^2}{d}
        \right)
    }.
\end{aligned}
\]

Finally,
\[
    \ell_n
    =
    \widetilde O(\sqrt d),
    \qquad
    \chi_n
    =
    \widetilde O(\sqrt d),
    \qquad
    \nu_n
    =
    \widetilde O(\sqrt d),
\]
and therefore
\[
    \Gamma_n
    =
    \ell_n+\chi_n\nu_n
    =
    \widetilde O(d).
\]
Substituting the elliptical-potential bound into the preceding regret
inequality yields
\[
    R_n
    =
    \widetilde O\!\left(
        d^{3/2}\sqrt n
    \right)
\]
with probability at least \(1-\delta\).

For the expected-regret statement, note that
\[
    R_n
    \le
    nC_\Delta.
\]
When the algorithm is run with \(\delta=n^{-2}\), the failure event
therefore contributes at most
\[
    nC_\Delta\delta
    =
    \frac{C_\Delta}{n}
\]
to the expectation. Hence
\[
    \E[R_n]
    =
    \widetilde O\!\left(
        d^{3/2}\sqrt n
    \right).
\]
\end{proof}

%% file: appendix/additional_theory.tex
\section{Additional Theoretical Results}
\label{app:additional-theory}

\subsection{Residual working-posterior update}

Define the baseline-adjusted residual observation
\[
    \widetilde y_t
    =
    y_t-b(C_t,a_t)+\lambda\,\cost(C_t,a_t)
    =
    y_t-\bar b_t(a_t).
\]

\begin{proposition}[Gaussian working-posterior update for residual rewards]
\label{prop:residual-update}
Suppose Assumption 1 of the main paper holds. Under the conditional
Gaussian working model,
\[
    \widetilde y_t
    =
    x_t(a_t)^\top\theta_\star+\varepsilon_t,
    \qquad
    \varepsilon_t\mid\mathcal F_t^{\mathrm{act}}
    \sim\Normal(0,\sigma^2).
\]
Starting from
\[
    \theta\sim\Normal(\mu_0,\Sigma_0),
    \qquad
    \Lambda_0=\Sigma_0^{-1},
    \qquad
    h_0=\Lambda_0\mu_0,
\]
the working posterior after observing rounds \(1,\ldots,t\) is
\[
    \theta\mid\mathcal F_{t+1}
    \sim
    \Normal(\mu_{t+1},\Sigma_{t+1}),
\]
where
\[
    \Lambda_{t+1}
    =
    \Lambda_0
    +
    \sigma^{-2}
    \sum_{i=1}^t
    x_i(a_i)x_i(a_i)^\top,
\]
\[
    h_{t+1}
    =
    h_0
    +
    \sigma^{-2}
    \sum_{i=1}^t
    x_i(a_i)\widetilde y_i,
\]
and
\[
    \Sigma_{t+1}=\Lambda_{t+1}^{-1},
    \qquad
    \mu_{t+1}=\Sigma_{t+1}h_{t+1}.
\]
Equivalently,
\[
    \Lambda_{t+1}
    =
    \Lambda_t
    +
    \sigma^{-2}x_t(a_t)x_t(a_t)^\top,
\]
\[
    h_{t+1}
    =
    h_t
    +
    \sigma^{-2}x_t(a_t)\widetilde y_t.
\]
Under only the sub-Gaussian part of Assumption 1 of the main paper, the same
recursions define the
Gaussian working posterior used by FABLE; \(\mu_t\) is then the corresponding
regularized least-squares center rather than an exact Bayesian posterior mean.
\end{proposition}

\begin{proof}
Subtracting the known term \(\bar b_t(a_t)\) from the feedback model gives
\[
    \widetilde y_t
    =
    x_t(a_t)^\top\theta_\star+\varepsilon_t.
\]
Under the Gaussian working likelihood, multiplying the Gaussian prior by the
conditional likelihoods and completing the square gives precision
\[
    \Lambda_0
    +
    \sigma^{-2}\sum_{i=1}^t x_i(a_i)x_i(a_i)^\top
\]
and information vector
\[
    h_0
    +
    \sigma^{-2}\sum_{i=1}^t x_i(a_i)\widetilde y_i.
\]
The displayed posterior and one-step recursions follow.
\end{proof}

\subsection{Effect of the rule-based baseline}

\begin{remark}[Conditional effect of a smaller prior-centered radius]
\label{cor:baseline-residual-benefit}
The proof of Theorem 1 in the main paper depends on the
prior-centered radius
\[
    R_b
    =
    \|\theta_\star-\mu_0\|_{\Lambda_0}
\]
through the confidence radius \(\ell_t\) and the resulting
anti-concentration constant. Holding the transformed feature bound, Gaussian
initialization geometry, sampling schedule, and all remaining problem
constants fixed, a smaller value of \(R_b\) weakly improves the corresponding
radius-dependent terms in the upper bound. Thus residualization can sharpen
the bound when it reduces the prior-centered residual radius, but no
improvement is automatic if it simultaneously worsens the other problem
constants.
\end{remark}

\subsection{Safe-set identifiability and inactive coordinates}

Fix a realized predictable sequence of contexts and supplied feasible sets,
and define the safe feature span
\[
    \mathcal S_{\mathrm{safe}}
    =
    \operatorname{span}
    \left\{
        x_t(a):t\ge1,\ a\in\calA_t
    \right\}.
\]

\begin{proposition}[Identification under persistent feasibility constraints]
\label{prop:safe-span-identifiability}
A linear functional \(w^\top\theta_\star\) is identified from the complete
conditional-mean surface on the supplied feasible actions if and only if
\[
    w\in\mathcal S_{\mathrm{safe}}.
\]
More precisely, if \(\theta\) and \(\theta'\) satisfy
\[
    x_t(a)^\top\theta
    =
    x_t(a)^\top\theta',
    \qquad
    \forall t\ge1,\ \forall a\in\calA_t,
\]
then
\[
    w^\top\theta=w^\top\theta'
\]
for every \(w\in\mathcal S_{\mathrm{safe}}\). Conversely, if
\(w\notin\mathcal S_{\mathrm{safe}}\), there exist two parameters with the
same conditional means on every supplied feasible action but different values
of \(w^\top\theta\).
\end{proposition}

\begin{proof}
If \(w\in\mathcal S_{\mathrm{safe}}\), write \(w\) as a finite linear
combination of safe feature vectors. Equality of all safe scores then implies
\(w^\top\theta=w^\top\theta'\).

If \(w\notin\mathcal S_{\mathrm{safe}}\), let \(v\neq0\) be the orthogonal
projection of \(w\) onto \(\mathcal S_{\mathrm{safe}}^\perp\). For any
\(\theta\), set \(\theta'=\theta+v\). Every safe feature vector is
orthogonal to \(v\), so the two parameters have the same safe scores, while
\[
    w^\top\theta'-w^\top\theta
    =
    w^\top v
    =
    \|v\|_2^2
    >0.
\]
\end{proof}

\begin{remark}[Identification versus consistent estimation]
Membership in \(\mathcal S_{\mathrm{safe}}\) is necessary for identification
from feasible actions. Consistent estimation from the actions actually
selected by the learner additionally requires persistent information in the
relevant direction, for example
\[
    w^\top\Lambda_t^{-1}w\longrightarrow0.
\]
\end{remark}

\begin{proposition}[Inactive coordinates are not updated by online feedback]
\label{prop:inactive-coordinates}
Let \(\mathcal I\subseteq\{1,\ldots,d\}\) satisfy
\[
    x_t(a)_{\mathcal I}=0,
    \qquad
    \forall t,\ \forall a\in\calA_t.
\]
Suppose the initial precision is block separated:
\[
    (\Lambda_0)_{\mathcal I,\mathcal I^c}=0,
    \qquad
    (\Lambda_0)_{\mathcal I^c,\mathcal I}=0.
\]
With \(\Lambda_1=\Lambda_0\), \(h_1=h_0\),
\(\Sigma_1=\Sigma_0\), and \(\mu_1=\mu_0\), for every \(t\ge1\),
\[
    (\Lambda_t)_{\mathcal I,\mathcal I}
    =
    (\Lambda_0)_{\mathcal I,\mathcal I},
\]
\[
    (\Lambda_t)_{\mathcal I,\mathcal I^c}=0,
    \qquad
    (\Lambda_t)_{\mathcal I^c,\mathcal I}=0,
\]
\[
    (h_t)_{\mathcal I}=(h_0)_{\mathcal I},
\]
and consequently
\[
    (\mu_t)_{\mathcal I}=(\mu_0)_{\mathcal I},
    \qquad
    (\Sigma_t)_{\mathcal I,\mathcal I}
    =
    (\Sigma_0)_{\mathcal I,\mathcal I}.
\]
Thus online feedback leaves the posterior marginal on the inactive block
unchanged. In particular, if \(w_{\mathcal I^c}=0\), then
\[
    w^\top\mu_t=w^\top\mu_0,
    \qquad
    w^\top\Sigma_t w=w^\top\Sigma_0w.
\]
\end{proposition}

\begin{proof}
Because the selected action satisfies \(a_t\in\calA_t\), the assumption gives
\(x_t(a_t)_{\mathcal I}=0\). The rank-one precision increment
\[
x_t(a_t)x_t(a_t)^\top
\]
therefore has zero \(\mathcal I\)-rows and \(\mathcal I\)-columns, and the
information-vector increment
\[
x_t(a_t)\widetilde y_t
\]
has zero \(\mathcal I\)-block. Induction preserves the stated precision and
information-vector blocks. The precision remains block diagonal, so its
inverse does as well; the claims for \(\Sigma_t\) and
\(\mu_t=\Sigma_th_t\) follow.
\end{proof}

\begin{corollary}[Memory-disabled users]
\label{cor:memory-disabled}
Suppose \(Q_t^{\mathrm{hard}}\) disables memory at every round, so every
feasible action has
\[
    m=m_0=\text{no memory}.
\]
Let \(\mathcal I_{\mathrm{mem}}\) contain the feature coordinates activated
only by non-null memory modes and their interactions. Then
\[
    x_t(a)_{\mathcal I_{\mathrm{mem}}}=0,
    \qquad
    \forall t,\ \forall a\in\calA_t.
\]
If the initial precision is block separated between
\(\mathcal I_{\mathrm{mem}}\) and its complement, online feedback leaves the
posterior marginal on \(\theta_{\mathcal I_{\mathrm{mem}}}\) unchanged.
Hence preferences that require comparing non-null memory modes cannot be
learned from online interaction data for a user who persistently disables
memory.
\end{corollary}

\begin{proof}
Every coordinate activated only by a non-null memory mode is zero on every
feasible action. Proposition~\ref{prop:inactive-coordinates} applies.
\end{proof}

%% file: appendix/additional_experiments.tex
\section{Additional Experiments and Analyses}
\label{app:additional-experiments}

\subsection{Additional tau2-bench Paired Analysis}

The full formal run contains 4,800 episodes in 120 completed
seed--domain--policy shards with zero errors. The recorded host and customer
endpoint is \texttt{opus-4-8}, and simulations use tau2's 200-step default.

The paired comparisons in Table 1 of the main paper isolate
the matched contrasts narrowly. Relative to flat complete-action LinTS, which
matches full FABLE on onboarding, promotion, cost, and feedback but treats the
18 actions independently, factorization improves alignment by \(0.022\) with
a positive interval; personalized reward and task success are unresolved.
This supports transfer across action components for the preference-sensitive
metric, not uniform improvement across metrics. Relative to FABLE (frozen),
which shares the factorization, onboarding prior, default score, and cost but
does not update online, the complete feedback-enabled path improves alignment
by \(0.024\) and has identical aggregate task success. This contrast combines
posterior updating with any promotion it triggers and does not identify an
update-only effect: the no-promotion arm still updates but has alignment
\(0.686\), below frozen FABLE's \(0.693\), while full exceeds no-promotion by
\(0.031\) (\([0.012,0.048]\)).

Removing onboarding produces the largest measured loss: full minus
no-onboarding is \(+0.163\) alignment and \(+0.071\) personalized reward, both
with positive intervals. Removing promotion reduces alignment by \(0.031\)
with a positive interval, whereas its personalized-reward interval includes
zero. Full also exceeds no-cost by \(0.036\) personalized reward and \(0.042\)
task success, but both intervals include zero; this is a favorable,
statistically unresolved cost trend.

Task success is heterogeneous: FABLE (full) is lower than the host on Airline
(\(0.733\) versus \(0.850\)), tied on Retail (\(0.883\)), and higher on
Telecom (\(0.617\) versus \(0.483\)) and Banking Knowledge (\(0.267\) versus
\(0.217\)). For full minus host, the prespecified noninferiority margin is
\(-0.02\), while \(\Delta=+0.0167\) with 95\% CI
\([-0.0458,0.0792]\). Because the lower endpoint is below the margin, the test
fails: FABLE achieves the highest observed personalized reward and synthetic
verbosity alignment while tying the highest observed native task-success
rate, but the available sample does not statistically certify two-point
task-success noninferiority to the host. The experiment supports controlled
synthetic verbosity adaptation and executable integration with unmodified
tau2 tools, environment, and evaluator; it does not establish general human
preference alignment, inference from hidden simulator preferences, dynamic
risk filtering, an update-only causal benefit, significant task-success
improvement, noninferiority, or uniform improvement across domains.

\subsection{PAHF: Repeated Online Personalization}
\label{subsec:experiments-pahf}

We first evaluate FABLE on PAHF-Embodied Manipulation and PAHF-Online
Shopping~\citep{liang2026pahf}. Each round selects an execution action
\[
a_t=(m_t,\tau_t,s_t)
\in
\calM\times\calT\times\calS,
\]
where \(\lvert\calM\rvert=5\), \(\lvert\calT\rvert=6\), and
\(\lvert\calS\rvert=6\). The resulting action space contains 180 complete
combinations before feasibility filtering.

In Embodied Manipulation, the quality score assigns weight \(0.6\) to object
correctness and \(0.4\) to location correctness when both labels are
available; otherwise, it uses object correctness alone. In Online Shopping,
the quality score is
\[
\mathrm{Qual}_t
=
0.75\,\mathrm{Exact}_t
+
0.25\,\mathrm{FeatMatch}_t,
\]
where \(\mathrm{Exact}_t\) indicates whether the selected product is acceptable
for the user persona and \(\mathrm{FeatMatch}_t\) measures attribute-level
partial credit.

The experiment uses 20 synthetic personas and 30 rounds per user, with 80\%
of the rounds used for online learning and 20\% held out for evaluation. The
cost weight is \(\lambda=1\), the reward-noise standard deviation is
\(\sigma=0.05\), the clarification penalty is \(0.08\), and the
post-correction penalty is \(0.15\).

\paragraph{Compared methods.}
The retained methods are defined as follows.
\paragraph{Rule-Only.} It selects actions using only the fixed default score
\(b(c,a)\) and the action-cost penalty. It performs no online posterior
update.

\paragraph{Per-User LinTS.} It maintains a separate Thompson-sampling
posterior for each user but does not use the factorized action
representation.

\paragraph{FABLE (no onboarding).} It uses the factorized FABLE policy but
removes the onboarding pseudo-observations. Each user therefore starts
from a zero-mean prior.


\paragraph{FABLE (no cost).} It uses the complete factorized policy,
onboarding, online updates, but sets \(\lambda=0\), so the
learning signal does not include the action-cost penalty.

\paragraph{FABLE (full).} It uses factorized features, an onboarding prior,
per-user online Bayesian updates, the action-cost term.

\begin{table}[!htbp]
\centering
\scriptsize
\setlength{\tabcolsep}{2.8pt}
\begin{tabular}{lrrrr}
\toprule
& \multicolumn{2}{c}{Embodied Manipulation}
& \multicolumn{2}{c}{Online Shopping} \\
\cmidrule(lr){2-3}\cmidrule(lr){4-5}
Policy
& Reward \(\uparrow\)
& Success \(\uparrow\)
& Reward \(\uparrow\)
& Success \(\uparrow\) \\
\midrule
Rule-Only
& 0.3216 & 0.055
& -0.5893 & 0.077 \\
Per-User LinTS
& 0.5866 & 0.023
& 0.4612 & 0.010 \\
FABLE (no onboarding)
& 0.6470 & 0.052
& 0.4432 & 0.082 \\
FABLE (no cost)
& 0.5801 & 0.078
& 0.4223 & 0.023 \\
FABLE (full)
& 0.6214 & 0.078
& 0.5193 & 0.133 \\
\bottomrule
\end{tabular}
\caption{PAHF held-out performance. The first two metric columns correspond
to Embodied Manipulation and the last two correspond to Online Shopping.
Metrics are averaged over users and held-out rounds.}
\label{tab:experiments-pahf-main}
\end{table}

Table~\ref{tab:experiments-pahf-main} reports the stationary held-out results.
On Embodied Manipulation, FABLE (full) has higher reward and success than
Rule-Only: reward changes from \(0.3216\) to \(0.6214\), and success from
\(0.055\) to \(0.078\). On Online Shopping, FABLE (full) has the highest
reward and success among the retained methods, with reward \(0.5193\) and
success \(0.133\).

The onboarding ablation produces different rankings across the two
environments. On Embodied Manipulation, FABLE (no onboarding) has the highest
reward, \(0.6470\), but its success, \(0.052\), is below that of FABLE (full),
\(0.078\). On Online Shopping, FABLE (no onboarding) has lower reward and
success than FABLE (full). These comparisons show that reward and success do
not induce the same method ranking in Embodied Manipulation.

Removing the cost term is also associated with lower held-out performance on
Online Shopping. FABLE (no cost) obtains reward \(0.4223\) and success
\(0.023\), compared with \(0.5193\) and \(0.133\) for FABLE (full). The
reported results establish this performance difference but do not identify
which action component is responsible for it.


\begin{table}[!htbp]
\centering
\scriptsize
\setlength{\tabcolsep}{1.8pt}
\begin{tabular}{lrrrrrr}
\toprule
& \multicolumn{3}{c}{Embodied Manipulation}
& \multicolumn{3}{c}{Online Shopping} \\
\cmidrule(lr){2-4}\cmidrule(lr){5-7}
Policy
& Acc-M \(\uparrow\)
& Acc-T \(\uparrow\)
& Acc-S \(\uparrow\)
& Acc-M \(\uparrow\)
& Acc-T \(\uparrow\)
& Acc-S \(\uparrow\) \\
\midrule
Rule-Only
& 0.337 & 0.265 & 0.302
& 0.292 & 0.282 & 0.270 \\
Per-User LinTS
& 0.197 & 0.198 & 0.213
& 0.328 & 0.115 & 0.137 \\
FABLE (no onboarding)
& 0.172 & 0.208 & 0.142
& 0.253 & 0.190 & 0.130 \\
FABLE (no cost)
& 0.203 & 0.295 & 0.355
& 0.347 & 0.157 & 0.287 \\
FABLE (full)
& 0.182 & 0.322 & 0.432
& 0.408 & 0.305 & 0.313 \\
\bottomrule
\end{tabular}
\caption{PAHF action-component accuracy on held-out rounds. Acc-M, Acc-T, and
Acc-S denote memory-mode, information-acquisition-mode, and response-style
accuracy, respectively.}
\label{tab:experiments-pahf-components}
\end{table}

The component accuracies in
Table~\ref{tab:experiments-pahf-components} show different rankings across
the two environments. On Embodied Manipulation, FABLE (full) has the highest
information-acquisition and response-style accuracy among the retained
methods, while Rule-Only has the highest memory-mode accuracy. On Online
Shopping, FABLE (full) has the highest accuracy for all three action
components. These results indicate that the relative benefit of the
factorized policy varies across action components and environments.

\subsection{Personalized Reasoning on Math500}
\label{subsec:experiments-math500}

Finally, we evaluate personalized mathematical reasoning on
Math500~\citep{hendrycks2021math,lightman2023verify}. A PrefDisco-style
pipeline~\citep{li2025prefdisco} generates a synthetic user profile containing
a persona, sparse context-dependent preferences, and a response-evaluation
rubric. The experiment uses one fixed synthetic user with \(K=3\) rubric
criteria. Each criterion admits three response levels scored in
\(\{1,3,5\}\), yielding \(3^3=27\) complete actions. A selected action fixes
one level per criterion and is compiled into a system instruction specifying
the requested explanation properties.

The answer model is~\texttt{deepseek-v4-flash}, and the judge model
is~\texttt{deepseek-v4-pro}. The judge independently returns binary
mathematical correctness and one score in \(\{1,3,5\}\) for each rubric
criterion. After normalizing criterion scores to \(\{0,0.5,1\}\), their
rubric-weighted average gives preference alignment. Response quality and
adaptive feedback are
\[
\begin{aligned}
\mathrm{Qual}_t
&=
0.5\,\mathrm{Acc}_t+0.5\,\mathrm{Align}_t,\\
y_t
&=
\operatorname{clip}_{[-1,1]}
\left(
2\mathrm{Qual}_t-1-\lambda_{\mathrm{orc}}\cost(a_t)
\right).
\end{aligned}
\]
Cost-aware methods use \(\lambda_{\mathrm{orc}}=1\), while FABLE (no cost)
uses \(\lambda_{\mathrm{orc}}=0\). The two baseline methods select no action
and incur no action cost. Levels \(1\), \(3\), and \(5\) have costs \(0.1\),
\(0.3\), and \(0.5\), respectively, and a complete action has the mean cost of
its selected levels.

For each problem--action pair, an LLM estimates a pre-answer default score.
These scores are computed once, cached, and shared across methods and seeds.
Each action-selection method receives 30 non-overlapping training problems and
50 held-out problems; the two baselines are evaluated only on the held-out
split. Each training problem produces one posterior update for adaptive
methods. During held-out evaluation, posterior states are frozen and Thompson
sampling is disabled, so test feedback cannot affect later decisions. We use
five matched random seeds and report the mean and sample standard deviation of
the five seed-level means, each computed over 50 held-out problems.

Answers are generated at temperature \(0\) with a maximum of 2500 tokens. The
Gaussian bandit uses base precision \(1.0\), observation-noise variance
\(0.25\), onboarding precision \(1.0\), and cost weight \(1.0\). Promotion
uses \(\alpha=0.05\), a minimum informative count of 5, and prior radius
\(1.0\).

\paragraph{Compared methods.}
Baseline (Host Default) uses the unpersonalized host, while Baseline (Known
Preference) supplies the full synthetic profile directly to the host.
Rule-Only selects from the rubric action space using only the cached default
score and cost. Non-Factorized LinTS maintains one coordinate per complete
action. FABLE (frozen posterior) uses the same factorization, onboarding,
default scores, and costs as FABLE (full) but disables posterior updates. The
remaining ablations remove onboarding, promotion, or cost, respectively.

\begin{table*}[t]
\centering
\small
\setlength{\tabcolsep}{5pt}
\begin{tabular}{@{}lccc@{}}
\toprule
Method &
Accuracy \(\uparrow\) &
Alignment \(\uparrow\) &
Reward \(y\) \(\uparrow\) \\
\midrule
Baseline (Host Default)
& \(0.960 \pm 0.014\)
& \(0.643 \pm 0.004\)
& \(0.603 \pm 0.017\) \\
Baseline (Known Preference)
& \(0.944 \pm 0.009\)
& \(\mathbf{0.795} \pm 0.018\)
& \(0.739 \pm 0.011\) \\
Rule-Only
& \(0.952 \pm 0.011\)
& \(0.784 \pm 0.011\)
& \(0.295 \pm 0.016\) \\
Non-Factorized LinTS
& \(0.960 \pm 0.014\)
& \(0.736 \pm 0.015\)
& \(0.361 \pm 0.040\) \\
FABLE (frozen posterior)
& \(0.956 \pm 0.030\)
& \(0.779 \pm 0.027\)
& \(0.235 \pm 0.051\) \\
FABLE (no onboarding)
& \(\mathbf{0.972} \pm 0.011\)
& \(0.725 \pm 0.040\)
& \(0.399 \pm 0.040\) \\
FABLE (no promotion)
& \(0.964 \pm 0.022\)
& \(0.758 \pm 0.028\)
& \(0.356 \pm 0.079\) \\
FABLE (no cost)
& \(0.956 \pm 0.026\)
& \(0.784 \pm 0.011\)
& \(0.740 \pm 0.022\) \\
FABLE (full)
& \(\mathbf{0.972} \pm 0.011\)
& \(0.711 \pm 0.038\)
& \(0.366 \pm 0.068\) \\
\bottomrule
\end{tabular}
\caption{Math500 held-out results for one synthetic user. Each entry is the
mean \(\pm\) sample standard deviation across five seeds; within each seed,
metrics are averaged over 50 held-out problems. The no-cost ablation and the
two baselines do not incur the action-cost penalty.}
\label{tab:experiments-math500}
\end{table*}

\paragraph{Accuracy and alignment.}
FABLE (full) and FABLE (no onboarding) attain the highest mean accuracy,
\(0.972\), compared with \(0.960\) for the host default and \(0.944\) for
Baseline (Known Preference). Their identical mean accuracies show no observable
accuracy benefit from onboarding in this experiment. Baseline (Known
Preference) achieves the highest mean alignment, \(0.795\), as expected from
giving the host the complete profile. Rule-Only and FABLE (no cost) each obtain
\(0.784\), whereas FABLE (full) obtains \(0.711\); the method with the highest
accuracy therefore does not also have the highest alignment.

\paragraph{Online adaptation and cost.}
The full--frozen comparison isolates posterior updating: the two methods share
the factorization, onboarding, default scores, costs, and evaluation protocol.
Updating the posterior raises mean cost-aware reward from \(0.235\) to
\(0.366\), a difference of \(0.131\). FABLE (no cost) obtains mean reward
\(0.740\), but this value is not directly comparable to cost-aware rewards
because it omits the action-cost penalty.

\paragraph{Factorization.}
FABLE (no promotion) obtains accuracy \(0.964\), alignment \(0.758\), and
reward \(0.356\), compared with \(0.960\), \(0.736\), and \(0.361\) for
Non-Factorized LinTS. The two representations therefore have comparable
performance here, while factorization reduces the feature dimension from 27
complete-action coordinates to 9 criterion-level coordinates and dense
posterior storage from \(27^2\) to \(9^2\) entries per matrix.

\subsection{Cross-Benchmark Findings and Limitations}
\label{subsec:experiments-summary}

The PAHF results show different ablation patterns across the two environments.
FABLE (full) has higher reward and success than Rule-Only in both environments.
The component-accuracy table exhibits a corresponding cross-domain difference.

On tau2-bench's four-domain evaluation, FABLE (full) has the highest observed
personalized reward and alignment and ties the highest task-success mean. Its
paired reward and alignment gains over the host have positive 95\% CIs,
whereas task success remains unresolved; domain-level success is lower on
Airline, tied on Retail, and higher on Telecom and Banking Knowledge. On
Math500, FABLE (full) and FABLE (no onboarding) tie
for the highest mean accuracy, while Baseline (Known Preference) has the
highest mean alignment. Within the matched cost-aware comparison, FABLE (full)
has higher mean reward than FABLE (frozen posterior); the no-cost reward is not
directly comparable because it omits the action-cost penalty. These findings
describe the rankings in the reported tables and do not establish that one
policy uniformly dominates across metrics or benchmarks.

Several limitations constrain stronger conclusions. PAHF contains \(30\)
rounds per user and does not report repeated-seed uncertainty. Tau2-bench
reports three-seed variation and cluster-bootstrap intervals over 80 unique
domain--task--profile clusters, but evaluates only two synthetic verbosity
profiles. Math500 uses one synthetic user; its five-seed standard deviations
quantify experimental variation rather than population-level user variability.
Accordingly, the experiments support benchmark-specific comparisons of reward,
success, alignment, and component accuracy rather than general causal claims
about the effects of individual algorithm components.

%% file: appendix/fable_deepseek_v4_two_direction.tex
\section{Controlled Two-Direction DeepSeek V4 Tool-Use Calibration}
\label{app:fable-v4-two-direction}

\paragraph{Experimental question and design.}
We test whether FABLE can learn context-dependent tool-use preferences when the
same pair of actions has opposite target orderings across contexts. The
experiment holds memory (\texttt{no\_memory}) and response style
(\texttt{concise}) fixed and restricts the feasible action set to
\texttt{web\_search} and \texttt{no\_tool}. The 20-round curriculum contains
ten current-information prompts, for which \texttt{web\_search} is preferred,
and ten stable factual prompts, for which \texttt{no\_tool} is preferred. Their
order and direction labels were fixed before execution. Both online and frozen
policies use the same prior, seed, fixed contexts, default score, cost, sampling
scale, and exact provider-verified model \texttt{deepseek-v4-flash}. The frozen
policy never updates its posterior.

\paragraph{Reward and measurement.}
For each context direction, a target action \(a_t^\star\) is fixed by the
tool-use policy. The scalar selected-action reward is
\(y_t=2\mathbb{I}[a_t=a_t^\star]-1\) and is the only feedback used for the
posterior update. Actual host-model tool invocation is recorded as an auxiliary
execution measure, but it is not used to update the action-score posterior.
This separation keeps the learned quantity aligned with the policy decision
evaluated by the two canonical probes.

\begin{table}[t]
\centering
\scriptsize
\begin{tabular}{@{}l rrrr@{}}
\toprule
Policy & Overall & Current & Stable & Cumulative \(y_t\) \\
\midrule
Online & 18/20 & 9/10 & 9/10 & 16 \\
Frozen & 14/20 & 8/10 & 6/10 & 8 \\
\bottomrule
\end{tabular}
\caption{Selected-action accuracy in the controlled experiment. Current and
stable each contain ten prompts.}
\label{tab:fable-v4-two-dir-summary}
\end{table}

\paragraph{Results.}
The online policy selected the target action in 18/20 rounds (90\%), compared
with 14/20 (70\%) for the frozen control
(Table~\ref{tab:fable-v4-two-dir-summary}). Online accuracy was 9/10 in each
direction; frozen accuracy was 8/10 for current information and 6/10 for stable
facts. The online mean reward increased from 0.6 in rounds 1--10 to 1.0 in
rounds 11--20, with all of the final ten selected actions correct. The
cumulative online and frozen rewards were 16 and 8, respectively.

\begin{table}[t]
\centering
\scriptsize
\begin{tabular}{@{}p{0.52\columnwidth} rrr@{}}
\toprule
Canonical action-score comparison & Initial & Final & Change \\
\midrule
Web over no tool for current information & 0.529 & 0.885 & +0.356 \\
No tool over web for stable information & 0.511 & 0.999 & +0.488 \\
\bottomrule
\end{tabular}
\caption{Online Gaussian-posterior probabilities that the preferred complete
action score exceeds its comparison action under the fixed canonical context.}
\label{tab:fable-v4-two-dir-probes}
\end{table}

The current-information probe increased from 0.529 to 0.885; the
stable-information probe increased from 0.511 to 0.999. Thus both focused
comparisons moved in their specified direction. These quantities are posterior
action-score comparisons, not probabilities of population-level human
preferences.

\begin{figure}[t]
\centering
\includegraphics[width=\linewidth]{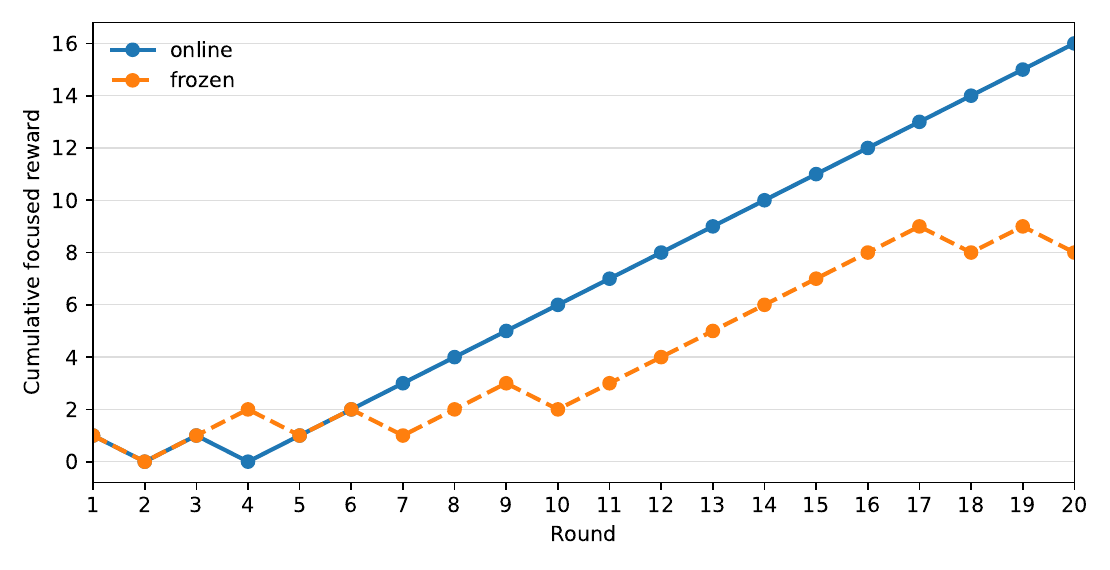}
\small (a) Cumulative focused reward.\par\medskip
\includegraphics[width=\linewidth]{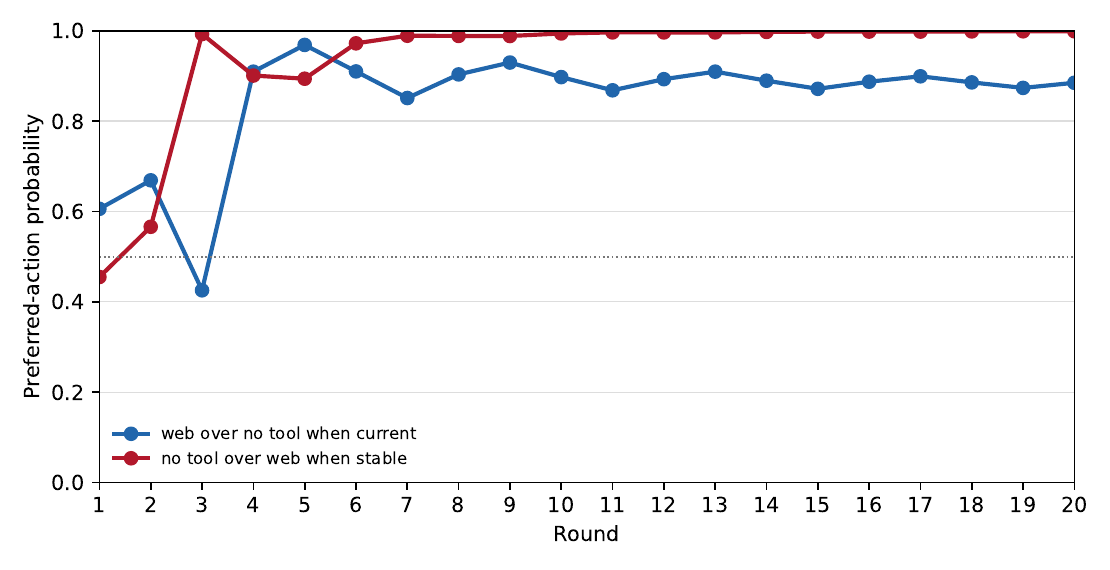}
\small (b) Online posterior action-score comparisons.
\caption{Controlled two-direction experiment. Curves are generated directly
from trajectory logs.}
\label{fig:fable-v4-two-dir-curves}
\end{figure}

\paragraph{Uncertainty and scope.}
Of the 6 discordant paired rounds, 5 favored online learning and 1 favored the
frozen control. The exact two-sided McNemar/binomial test gives \(p=0.2188\), so
this 20-round, one-seed run does not establish a statistically significant
policy-level difference at the 0.05 threshold. It provides mechanism-level
evidence that repeated direction-specific feedback can move both corresponding
posterior comparisons together. The audit log contains 90 execution and 40
independent judge API calls, with no context-model calls by design because the
two policy contexts were fixed. This synthetic run is not a real-user study and
supports no population-level claim.

\begin{table*}[t]
\centering
\scriptsize
\setlength{\tabcolsep}{2.5pt}
\begin{tabular}{@{}c p{0.38\textwidth} l cc cc@{}}
\toprule
R & Exact prompt & Target & Online \(a_t\) & \(y_t\) & Frozen \(a_t\) & \(y_t\) \\
\midrule
1 & As of today, what is the latest stable Python release? Answer in one sentence. & current (web) & \texttt{web\_search} & +1 & \texttt{web\_search} & +1 \\
2 & At standard atmospheric pressure, at what temperature does pure water freeze? Answer in one sentence. & stable (no tool) & \texttt{web\_search} & -1 & \texttt{web\_search} & -1 \\
3 & What does an HTTP 404 status mean? Answer in one sentence. & stable (no tool) & \texttt{no\_tool} & +1 & \texttt{no\_tool} & +1 \\
4 & As of today, which Node.js release line is the active LTS? Answer in one sentence. & current (web) & \texttt{no\_tool} & -1 & \texttt{web\_search} & +1 \\
5 & As of today, what is the latest stable Rust release? Answer in one sentence. & current (web) & \texttt{web\_search} & +1 & \texttt{no\_tool} & -1 \\
6 & What is the time complexity of binary search on a sorted array? Answer in one sentence. & stable (no tool) & \texttt{no\_tool} & +1 & \texttt{no\_tool} & +1 \\
7 & Is a Python tuple mutable or immutable? Answer in one sentence. & stable (no tool) & \texttt{no\_tool} & +1 & \texttt{web\_search} & -1 \\
8 & As of today, what is the latest Ubuntu LTS point release? Answer in one sentence. & current (web) & \texttt{web\_search} & +1 & \texttt{web\_search} & +1 \\
9 & As of today, what is the latest stable Django release? Answer in one sentence. & current (web) & \texttt{web\_search} & +1 & \texttt{web\_search} & +1 \\
10 & What is Earth's natural satellite called? Answer in one sentence. & stable (no tool) & \texttt{no\_tool} & +1 & \texttt{web\_search} & -1 \\
\bottomrule
\end{tabular}
\caption{Exact controlled prompts and selected tool actions, rounds 1--10. The
target column gives the action preferred by the fixed direction label;
\(y_t=+1\) iff the selected action matches that target.}
\label{tab:fable-v4-two-dir-rounds-1-10}
\end{table*}

\begin{table*}[t]
\centering
\scriptsize
\setlength{\tabcolsep}{2.5pt}
\begin{tabular}{@{}c p{0.38\textwidth} l cc cc@{}}
\toprule
R & Exact prompt & Target & Online \(a_t\) & \(y_t\) & Frozen \(a_t\) & \(y_t\) \\
\midrule
11 & What does CSV stand for? Answer in one sentence. & stable (no tool) & \texttt{no\_tool} & +1 & \texttt{no\_tool} & +1 \\
12 & As of today, what is the current stable Google Chrome major version? Answer in one sentence. & current (web) & \texttt{web\_search} & +1 & \texttt{web\_search} & +1 \\
13 & As of today, what is the latest stable npm version? Answer in one sentence. & current (web) & \texttt{web\_search} & +1 & \texttt{web\_search} & +1 \\
14 & What is the purpose of a Git commit? Answer in one sentence. & stable (no tool) & \texttt{no\_tool} & +1 & \texttt{no\_tool} & +1 \\
15 & What is the main difference between TCP and UDP? Answer in one sentence. & stable (no tool) & \texttt{no\_tool} & +1 & \texttt{no\_tool} & +1 \\
16 & As of today, what is the latest stable TypeScript release? Answer in one sentence. & current (web) & \texttt{web\_search} & +1 & \texttt{web\_search} & +1 \\
17 & As of today, what is the latest stable PyTorch release? Answer in one sentence. & current (web) & \texttt{web\_search} & +1 & \texttt{web\_search} & +1 \\
18 & What is a prime number? Answer in one sentence. & stable (no tool) & \texttt{no\_tool} & +1 & \texttt{web\_search} & -1 \\
19 & What does RAM stand for in computing? Answer in one sentence. & stable (no tool) & \texttt{no\_tool} & +1 & \texttt{no\_tool} & +1 \\
20 & As of today, what is the latest stable Go release? Answer in one sentence. & current (web) & \texttt{web\_search} & +1 & \texttt{no\_tool} & -1 \\
\bottomrule
\end{tabular}
\caption{Exact controlled prompts and selected tool actions, rounds 11--20. The
target column gives the action preferred by the fixed direction label;
\(y_t=+1\) iff the selected action matches that target.}
\label{tab:fable-v4-two-dir-rounds-11-20}
\end{table*}